\documentclass[12pt,english]{article}

\usepackage{algorithm}
\usepackage{booktabs}
\usepackage{graphicx}
\usepackage{amsmath}
\usepackage{amssymb}
\usepackage{latexsym}
\usepackage{mathtools}
\usepackage{csquotes}
\usepackage{crop}
\usepackage{algorithmic,algorithm}
\usepackage{multirow}
\usepackage{bm}
\usepackage{bbm}
\usepackage{enumerate}
\usepackage{url}
\usepackage{array}
\usepackage{paralist}
\usepackage{diagbox}
\usepackage{wasysym}
\usepackage{booktabs}
\usepackage[dvipsnames]{xcolor}
\usepackage[colorlinks = true, pdfstartview = FitV, linkcolor = blue, citecolor = blue, urlcolor = blue]{hyperref}
\usepackage{caption}
\usepackage{subcaption}
\usepackage{graphicx}
\usepackage{pdfpages}

\usepackage{listings}

\usepackage{rotating}

\usepackage[capitalise]{cleveref}
\crefname{equation}{}{}
\crefname{figure}{Figure}{Figures}
\creflabelformat{equation}{\textup{(#2#1#3)}}
\crefname{assumption}{Assumption}{Assumptions}
\crefname{condition}{Condition}{Conditions}
\usepackage{xspace}

\usepackage{fullpage}
\usepackage{multirow}

\usepackage[sort,nocompress]{cite}
\usepackage{arydshln}
\usepackage{enumitem}
\setlist[enumerate,1]{leftmargin=*,wide=0em, noitemsep,nolistsep, label = {\bfseries \arabic*.}}
\setlist[itemize,1]{noitemsep,nolistsep}
\usepackage{pifont}% http://ctan.org/pkg/pifont
\makeatletter
\newcommand*{\transpose}{%
	{\mathpalette\@transpose{}}%
}
\newcommand*{\@transpose}[2]{%
	\raisebox{\depth}{$\m@th#1\intercal$}%
}
\makeatother
\renewcommand {\AA}  { {\bm{A}} }
\newcommand {\BB}  { {\bm{B}} }
\newcommand {\CC}  { {\bm{C}} }

\newcommand {\HH}  { {\bm{H}} }

\newcommand {\MMM}  { {\bm{M}} }
\newcommand {\RR}  { {\bm{R}} }

\newcommand {\QQ}  { {\bm{Q}} }
\newcommand {\bSS}  { {\bm{S}} }

\newcommand {\XX}  { {\bm{X}} }
\newcommand{\eye}{\bm{I}}

\newcommand {\zz}  { {\bm z} }

\newcommand {\xx}  { {\bm x} }

\newcommand {\rr}  { {\bm r} }
\newcommand {\res}  { {\bm r} }
\renewcommand {\aa}  { {\bm a} }

\newcommand {\vv}  { {\bm v} }

\newcommand {\bb}  { {\bm b} }

\newcommand {\ee}  { {\bm e} }

\renewcommand{\Pr}{\hbox{\bf{Pr}}}

\newcommand*\xbar[1]{%
	\hbox{%
		\vbox{%
			\hrule height 0.5pt % The actual bar
			\kern0.5ex%         % Distance between bar and symbol
			\hbox{%
				\kern-0.1em%      % Shortening on the left side
				\ensuremath{#1}%
				\kern-0.1em%      % Shortening on the right side
			}%
		}%
	}%
} 

\definecolor{forestgreen}{rgb}{0.13, 0.55, 0.13}

\definecolor{amber}{rgb}{1.0, 0.75, 0.0}

\definecolor{bananayellow}{rgb}{.8, 0.6, 0}

\newcounter{comment}

\usepackage{amsthm}
\usepackage[framemethod=TikZ]{mdframed}

\mdfdefinestyle{theoremstyle}{%
	linewidth = 1pt,%
	roundcorner = 10pt,%
	leftmargin = 0,%
	rightmargin = 0,%
	backgroundcolor = cyan!3,%
	outerlinecolor = magenta!70!black,%
	splittopskip = \topskip,%
	ntheorem = true,%
}
\mdtheorem[style=theoremstyle]{claim}{Claim}

\newmdtheoremenv[%
linewidth = 1pt,%
roundcorner = 10pt,%
leftmargin = 0,%
rightmargin = 0,%
backgroundcolor = green!3,%
outerlinecolor = blue!70!black,%
splittopskip = \topskip,%
ntheorem = true,%
]{theorem}{Theorem}

\newmdtheoremenv[%
linewidth = 1pt,%
roundcorner = 10pt,%
leftmargin = 0,%
rightmargin = 0,%
backgroundcolor = green!3,%
outerlinecolor = blue!70!black,%
splittopskip = \topskip,%
ntheorem = true,%
]{corollary}{Corollary}

\newmdtheoremenv[%
linewidth = 1pt,%
roundcorner = 10pt,%
leftmargin = 0,%
rightmargin = 0,%
backgroundcolor = green!3,%
outerlinecolor = blue!70!black,%
splittopskip = \topskip,%
ntheorem = true,%
]{lemma}{Lemma}

\newmdtheoremenv[%
linewidth = 1pt,%
roundcorner = 10pt,%
leftmargin = 0,%
rightmargin = 0,%
backgroundcolor = yellow!3,%
outerlinecolor = blue!70!black,%
splittopskip = \topskip,%
ntheorem = true,%
]{definition}{Definition}

\newmdtheoremenv[%
linewidth = 1pt,%
roundcorner = 10pt,%
leftmargin = 0,%
rightmargin = 0,%
backgroundcolor = green!3,%
outerlinecolor = blue!70!black,%
splittopskip = \topskip,%
ntheorem = true,%
]{proposition}{Proposition}

\newmdtheoremenv[%
linewidth = 1pt,%
roundcorner = 10pt,%
leftmargin = 0,%
rightmargin = 0,%
backgroundcolor = green!3,%
outerlinecolor = blue!70!black,%
splittopskip = \topskip,%
ntheorem = true,%
]{condition}{Condition}

\newmdtheoremenv[%
linewidth = 1pt,%
roundcorner = 10pt,%
leftmargin = 0,%
rightmargin = 0,%
backgroundcolor = green!3,%
outerlinecolor = blue!70!black,%
splittopskip = \topskip,%
ntheorem = true,%
]{assumption}{Assumption}

\theoremstyle{definition}
\newmdtheoremenv[%
linewidth = 1pt,%
roundcorner = 10pt,%
leftmargin = 0,%
rightmargin = 0,%
backgroundcolor = blue!3,%
outerlinecolor = blue!70!black,%
splittopskip = \topskip,%
ntheorem = true,%
]{example}{Example}

\theoremstyle{definition}
\newmdtheoremenv[%
linewidth = 1pt,%
roundcorner = 10pt,%
leftmargin = 0,%
rightmargin = 0,%
backgroundcolor = red!3,%
outerlinecolor = blue!70!black,%
splittopskip = \topskip,%
ntheorem = true,%
]{remark}{Remark}

\usepackage{tikz}
\usepackage{xparse}% So that we can have two optional parameters

\NewDocumentCommand\DownArrow{O{2.0ex} O{black}}{%
	\mathrel{\tikz[baseline] \draw [<-, line width=0.5pt, #2] (0,0) -- ++(0,#1);}
}

\usepackage{listings} % to inser code

\definecolor{mygreen}{rgb}{0,0.6,0}
\definecolor{mygray}{rgb}{0.5,0.5,0.5}
\definecolor{mymauve}{rgb}{0.58,0,0.82}

\newcommand*\vnorm[1]{\left\| #1\right\|}

\newcommand*\bigO[1]{\mathcal O\left( #1\right)}

\usepackage{dsfont}

\usepackage[latin1]{inputenc}
\usepackage{amsmath}
\usepackage{amsfonts}
\usepackage{amssymb}
\usepackage{makeidx}
\usepackage{graphicx}
\usepackage{caption}
\usepackage{subcaption}
\usepackage{framed}
\usepackage{booktabs,array}
\usepackage{xcolor}

\allowdisplaybreaks

\begin{document}
	\title{\Large SALSA: Sequential Approximate Leverage-Score Algorithm \\
            with Application in Analyzing Big Time Series Data}
            
	\author{
		Ali Eshragh\thanks{Carey Business School, Johns Hopkins University, MD, USA, and International Computer Science Institute, Berkeley, CA, USA. Email:  \tt{ali.eshragh@jhu.edu}}
		\and
		Luke Yerbury\thanks{School of Information and Physical Sciences, University of Newcastle, NSW, Australia. Email:  \tt{luke.yerbury@newcastle.edu.au}}
		\and  
	Asef Nazari\thanks{School of Information Technology, Deakin University, Australia. Email:  \tt{asef.nazari@deakin.edu.au}} 
	\and
    Fred Roosta\thanks{School of Mathematics and Physics, University of Queensland, Australia, ARC Training Centre for Information Resilience, the University of Queensland, Australia, and International Computer Science Institute, Berkeley, CA, USA. Email:  \tt{fred.roosta@uq.edu.au}} 
    \and
 		Michael W. Mahoney\thanks{Department of Statistics, University of California at Berkeley, USA; Lawrence Berkeley National Laboratory, Berkeley, USA; and International Computer Science Institute, Berkeley, CA, USA. Email: \tt{mmahoney@stat.berkeley.edu}}
	}
	\date{}
	\maketitle
	
	%-----------------------
	% Abstract
	%-----------------------

\begin{abstract}
We develop a new efficient sequential approximate leverage score algorithm, SALSA, using methods from randomized numerical linear algebra (RandNLA) for large matrices. We demonstrate that, with high probability, the accuracy of SALSA's approximations is within $(1 + \bigO{\varepsilon})$ of the true leverage scores. In addition, we show that the theoretical computational complexity and numerical accuracy of SALSA surpass existing approximations. These theoretical results are subsequently utilized to develop an efficient algorithm, named LSARMA, for fitting an appropriate ARMA model to large-scale time series data. Our proposed algorithm is, with high probability, guaranteed to find the maximum likelihood estimates of the parameters for the true underlying ARMA model. Furthermore, it has a worst-case running time that significantly improves those of the state-of-the-art alternatives in big data regimes. Empirical results on large-scale data strongly support these theoretical results and underscore the efficacy of our new approach.
\end{abstract}

	%-----------------------
	% Introduction
	%-----------------------
	
	\section{Introduction}
\label{sec:introduction}

Leverage-scores are a statistical concept for measuring the influence of individual observations on the overall fitting of a model. They were initially used in the context of ordinary least squares ($\mathtt{OLS}$) regression, but the concept was extended to other statistical models too. In that context, data points with high leverage-scores can strongly affect the estimated coefficients of the regression model, potentially leading to skewing the model's predictions. Nowadays, leverage-scores play multiple roles across numerous domains within data analysis and machine learning. One prime application of leverage-scores is in detecting outliers and anomalies \cite{sharan2018efficient}, where they play an important role in identifying strange data points through the assessment of their corresponding leverage-score magnitudes. Furthermore, leverage-scores are useful in the context of dimensionality reduction \cite{liberty2007randomized} and the creation of sparse models \cite{kutyniok2014sparse}, because they help in the selection of the most relevant features or dimensions, consequently facilitating the development of more interpretable models. In addition, they are used in recommendation systems, where they are employed to propose items showing a strong influence on users \cite{cella2018efficient}. These applications highlight the multifaceted application of leverage-scores, proving their significance across a spectrum of data analysis and machine learning tasks.

In dealing with large-scale computations, especially in relation to huge data matrices and big $\mathtt{OLS}$ problems, randomized numerical linear algebra ($\mathtt{RandNLA}$) has appeared as a highly effective approach \cite{mahoneyRandomizedAlgorithmsMatrices2011,RandNLA_PCMIchapter_chapter,DM16_CACM,DM21_NoticesAMS,randlapack_book_v2_arxiv}. This field exploits various random sub-sampling and sketching methods to manage such scenarios. Essentially, $\mathtt{RandNLA}$ involves the random yet purposeful compression of the underlying data matrix into a smaller version while preserving many of its original characteristics. This compression allows a significant portion of the resource-intensive calculations to be executed using size-reduced matrices. \cite{mahoneyRandomizedAlgorithmsMatrices2011} and \cite{woodruffSketchingToolNumerical2014} have provided comprehensive insights into $\mathtt{RandNLA}$ subroutines and their diverse applications. Furthermore, implementations of algorithms based on these principles have demonstrated superior performance over state-of-the-art numerical methods \cite{avron2010blendenpik,meng2014lsrn,yang2015implementing}.

Condensing a sizable matrix into a computationally manageable form through sub-sampling requires an appropriate sampling technique. Although uniform sampling methods are straightforward, they prove significantly less effective when confronted with data non-uniformity, such as outliers. In such scenarios, non-uniform yet still independently and identically distributed (i.i.d.) sampling schemes, especially sampling based on leverage-score (\cite{drineas2012fast}), play an important role. They not only provide robust worst-case theoretical guarantees, but also facilitate the creation of high-quality numerical implementations. 

In particular, when dealing with time series data, leverage-score-based sampling methods show a high level of efficiency in estimating the parameters of $\mathtt{ARMA}$ models. However, computing these leverage-scores can be as computationally expensive as solving the original $\mathtt{OLS}$ problems. In this context, employing the time series model's inherent structure to estimate the leverage-scores appears to be a crucial determinant for the development of efficient algorithms in time series analysis. We explore this approach within the framework of $\mathtt{ARMA}$ models in this work.

This study offers the following contributions:
\begin{itemize}
    \item [(a)] Applying $\mathtt{RandNLA}$ methods to develop $\mathtt{SALSA}$, an efficient recursive algorithm for estimating leverage scores of larege matrices. 
    \item [(b)] Establishing a high-probability theoretical relative error bound for leverage-score estimation.
    \item [(c)] Utilizing \texttt{SALSA} to develop $\mathtt{LSARMA}$, a novel algorithm for estimating parameters of the $\mathtt{ARMA}$ model for large-scale time series data.
    \item [(d)] Providing both theoretical and empirical demonstrations of the efficiency of the proposed algorithms using both real-world and synthetic big data.
\end{itemize}

We structure the rest of this paper as follows: In \cref{sec:background}, we provide essential background information about $\mathtt{RandNLA}$, leverage-score-based sampling, and a brief literature review. \cref{chap_theo} presents the theoretical insights into recursive methods for the estimation and exact computation of leverage-scores, which form the basis for $\mathtt{SALSA}$. \cref{sec:SALSA_numerical_results} gives the numerical results obtained by applying $\mathtt{SALSA}$ for estimating leverage-scores in large synthetic and real-data matrices. 
In \cref{sec:lsarma}, we focus on the application of $\mathtt{SALSA}$ in handling big time series data, introducing a novel algorithm known as $\mathtt{LSARMA}$. Finally, \cref{Sec:Proofs} contains technical lemmas and their corresponding proofs.
%Finally, \cref{sec:conclusion} concludes the paper, while \cref{Sec:Proofs} contains technical lemmas and their corresponding proofs.

%------------------------
\subsubsection*{Notation}
%------------------------

In this paper, we adopt a notation scheme for vectors and matrices. Specifically, bold lowercase letters ($\vv$) represent vectors, uppercase letters ($\MMM$) represent matrices, and regular lowercase letters ($c$) denote scalars. All vectors are column vectors, and we indicate the transpose of a vector as $\vv^{\transpose}$. For real matrices $\MMM$, we use $\| \MMM \|$ to imply the $\ell_2$ norm. 
Following the conventions of MATLAB, $\MMM (i,:)$ and $\MMM (:,j)$ refer to the $i^{th}$ row and $j^{th}$ column of matrix $\MMM$, respectively. Furthermore, $\vv(i)$ denotes the $i^{th}$ component of vector $\vv$. We represent a vector with a value of one in the $i^{th}$ component and zeros elsewhere as $\ee_i$. Lastly, we represent the identity matrix of size $n \times n$ as $\eye_n$, and the number of nonzero components of matrix $\MMM$ is denoted by $\textnormal{nnz}(\MMM)$.
	
	%-----------------------
	% Background
	%-----------------------

	\section{Background}
\label{sec:background}

%---------------------------------------
\subsection{Randomized Numerical Linear Algebra}
\label{sec:RandNLA}
%---------------------------------------

This section provides an introduction to key $\mathtt{RandNLA}$ concepts, which will play a fundamental role in shaping the development of $\mathtt{SALSA}$ and its associated theoretical and numerical outcomes.  
$\mathtt{RandNLA}$ is an evolving field of research \cite{mahoneyRandomizedAlgorithmsMatrices2011,RandNLA_PCMIchapter_chapter,DM16_CACM,DM21_NoticesAMS,randlapack_book_v2_arxiv,eshragh_2020_toeplitz}, where the strategic integration of randomization techniques into traditional numerical algorithms enables the creation of more efficient algorithms suitable for linear algebra challenges in dealing with big data. 
These challenges pertain to a spectrum of critical tasks in linear algebra, including large matrix multiplication, low-rank matrix approximation, and solving the least-squares approximation problem \cite{RandNLA_PCMIchapter_chapter}. 
Although these fundamental problems are well studied, the area of $\mathtt{RandNLA}$ offers innovative approaches that improve traditional exact methods, promising more effective solutions in terms of both speed and memory requirements when dealing with large-scale problems. 
For example, consider an over-determined $\mathtt{OLS}$ problem,
\begin{equation}
\min_{\xx} \| \AA_{m,n} \xx - \bb \|^2,
\label{eqn:OLS_problem}
\end{equation}
where $\AA_{m,n} \in \mathbb{R}^{m \times n}$, $\xx \in \mathbb{R}^{n}$ and $\bb \in \mathbb{R}^{m}$. 

Solving \cref{eqn:OLS_problem} using the normal equations necessitates $\bigO{mn^2+ \frac{n^3}{3}}$ flops, 
employing QR factorization with Householder reflections reduces this requirement to $\bigO{mn^2-n^3}$ flops, and 
solving via the singular value decomposition increases it to $\bigO{mn^2+n^3}$ flops, 
as detailed in \cite{golubMatrixComputations2013}. 
Additionally, iterative solvers such as $\mathtt{LSQR}$ \cite{paigeLSQRAlgorithmSparse1982}, $\mathtt{LSMR}$ \cite{fongLSMRIterativeAlgorithm2011}, and $\mathtt{LSLQ}$ \cite{estrinLSLQIterativeMethod2019} require roughly $\bigO{mnc}$ flops after $c$ iterations. However, for large-scale problems with substantially large values of $m$ and $n$, the computational complexity of these algorithms worsens significantly. To mitigate this issue, when a controlled level of error bounds is acceptable, randomized approximation algorithms can substantially enhance computational efficiency by providing more favorable time-complexity outcomes.

$\mathtt{RandNLA}$ proposes a solution for addressing the problem presented in \cref{eqn:OLS_problem} by approximating the data matrix $\AA_{m,n}$ through the creation of an appropriately compressed matrix. This approximation is achieved by selecting $s$ rows from the matrix $\AA_{m,n}$ (where $n \leq s \ll m$) with replacement, guided by a specified discrete probability distribution defined over the rows. To formalize this process, we denote the resulting smaller matrix as $\bSS_{s,m} \AA_{m,n}$, with $\bSS_{s,m}$ belonging to $\mathbb{R}^{s \times m}$. This matrix is commonly referred to as a sketching or sampling matrix, and its construction is based on a predefined sampling procedure. The algorithms mentioned earlier for solving \cref{eqn:OLS_problem} can subsequently be employed to solve the modified problem, denoted as
\begin{equation}
\min_{\xx} \| \bSS_{s,m} \AA_{m,n} \xx - \bSS_{s,m} \bb \|^2.
\label{eqn:reduced_OLS_problem}
\end{equation}
In this problem, the originally larger dimension $m$ has been replaced with a significantly smaller value, denoted as $s$. Note that matrix multiplication within \cref{eqn:reduced_OLS_problem} is computationally expensive. Therefore, algorithms applying this approach are not typically implemented exactly in this manner. However, this formulation provides a valuable tool for theoretical analysis of such algorithms, allowing us to understand their worst-case behavior and performance characteristics.

To ensure the effectiveness of this approximation, the choice of the sketching matrix $\bSS_{s,m}$ must be decided carefully. In fact, the sketching matrix should be designed in a way that guarantees the achievement of an $\ell_2$ subspace embedding. This property becomes crucial when we permit a small failure probability $\delta$, where $0 < \delta < 1$, and the value of $s$ is sufficiently large. We can then establish the following relationship:
\begin{equation}
\label{eqn:subspace_embedding}
\Pr \left( \vnorm{\AA_{m,n}\xx^{\star} - \bb}^{2} \leq \vnorm{\AA_{m,n}\xx_{s}^{\star} - \bb}^{2} \leq \left(1+\bigO{\varepsilon}\right) \vnorm{\AA_{m,n}\xx^{\star} - \bb}^{2} \right) \geq 1-\delta,
\end{equation}
where $\xx^{\star}$ represents the solution to \cref{eqn:OLS_problem}, $\xx_{s}^{\star}$ represents the solution to \cref{eqn:reduced_OLS_problem}, and $0 <  \varepsilon< 1$.
Subspace embeddings are fundamental to RandNLA theory and practice and were first introduced in data-aware form for least squares regression in \cite{DMM06} (and immediately extended to low-rank approximation in \cite{DMM08_CURtheory_JRNL,CUR_PNAS}, where the connection with leverage-scores was made explicit).
They were first used in a data-oblivious form in \cite{Sarlos06,drineasFasterLeastSquares2011} and were popularized in data-oblivious form by \cite{woodruffSketchingToolNumerical2014}. 

The formulation of the sketching matrix $\bSS_{s,m}$ can be executed in various ways, broadly categorized as either data-oblivious or data-aware methods. Data-oblivious constructions involve techniques such as uniform sampling and random projection. Although uniform sampling is a fast process requiring constant time, it performs poorly when the rows of $\AA_{m,n}$ exhibit non-uniformity (outliers, as quantified by having large leverage-scores). To achieve the desired subspace embedding property, $s$ should be within $\bigO{m}$, which essentially defeats the purpose of subsampling. On the other hand, random projection methods, such as the Subsampled Randomized Hadamard Transform ($\mathtt{SRHT}$), aim to mitigate non-uniformity by projecting data into a subspace where uniform sampling becomes more effective. Details regarding these methods, and the required sample sizes to achieve \cref{eqn:subspace_embedding} in the data-oblivious setting, can be found in \cite{woodruffSketchingToolNumerical2014}.

Data-aware constructions, conversely, take into account the non-uniformity of the data when constructing a non-uniform sampling distribution over the rows of $\AA_{m,n}$. Distributions employing statistical leverage-scores to describe the significance of individual rows have been demonstrated to enhance the theoretical guarantees of matrix algorithms \cite{drineas2012fast} and prove highly effective in practical implementations \cite{mahoneyRandomizedAlgorithmsMatrices2011}, and distributions depending on generalizations of these scores lead to still further improvement in RandNLA theory and practice \cite{DM21_NoticesAMS}. An interesting observation is that the most effective random projection algorithms, including $\mathtt{SRHT}$, transform the data in such a manner that the leverage-scores become approximately uniformly distributed, making uniform sampling a suitable choice \cite{drineasFasterLeastSquares2011,drineas2012fast,mahoneyRandomizedAlgorithmsMatrices2011}.

%---------------------------------------
\subsection{Leverage-score-based Sampling}
\label{sec:lsbs}
%---------------------------------------
From a statistical perspective in the context of linear regression, leverage-scores are used to identify influential data points in a dataset. They help assess the impact of individual data points on the coefficients of a regression model. In addition, the magnitude of leverage-scores can be used in identifying anomalies and outliers in conjunction with their residuals \cite{rousseeuwRobustStatisticsOutlier2011}. For a matrix, a leverage-score associated with a row indicates the significance of that row in spanning its row space. This intuition is obtained from the fact that the rank of a matrix can be calculated as the sum of its leverage-scores \cite{chen2015completing}. This property makes leverage-scores a valuable tool in sketching techniques because they provide a sampling mechanism that increases the likelihood of selecting rows that contain \enquote{important}  information. As a consequence, this improves the ability of a subsampled matrix to preserve the information contained in the full matrix.   

For a matrix $\AA_{m,n}$ with $m \geq n$, the leverage-score associated with the $i^{\text{th}}$ row of $\AA_{m,n}$ is defined as 
\begin{equation*}
\ell_{m,n}(i) := \| \QQ_{m,m}(i,:) \|^2,
\end{equation*}
where $\QQ_{m,m}$ is \emph{any} orthogonal matrix with the property $\text{range}(\QQ_{m,m}) = \text{range}(\AA_{m,n})$. This definition is essential because it illustrates that leverage-scores are well defined and do not rely on a specific choice of the orthogonal matrix $\QQ_{m,m}$. A more practical definition of leverage-scores is derived from the concept of a hat matrix. The hat matrix for the matrix $\AA_{m,n}$ is defined as 
\begin{equation}
\HH_{m,m} := \AA_{m,n} (\AA_{m,n}^\transpose \AA_{m,n} )^{-1} \AA_{m,n}^\transpose.   
\label{HatMatrix}
\end{equation}
In this setting, the leverage-score of the $i^{\text{th}}$ row of $\AA_{m,n}$ is the $i^{\text{th}}$ diagonal entry of the hat matrix, which is expressed as
\begin{equation}
    \ell_{m,n} (i) := \ee_i^\transpose \HH_{m,m} \ee_i, \quad \text{for } i=1,\ldots,m. 
    \label{eqn:Def1_lev_scores}
\end{equation}
The leverage-scores of rows of a matrix can be used to construct a distribution for sampling over the rows of that matrix. In the following definition for $\AA_{m,n}$, 
\begin{equation}
\pi_{m,n}(i) := \frac{\ell_{m,n}(i)}{n} , \quad \text{for } i = 1,\ldots,m,
\label{eqn:sampling_dist}
\end{equation}
as $\ell_{m,n}(i)>0 \; \forall i \;$ and $\sum\limits_{i=1}^{m} \ell_{m,n}(i) = n$. We conclude that  $\pi_{m,n}(i)>0 \; \forall i \;$ and $\sum\limits_{i=1}^{m} \pi_{m,n}(i) = 1$, and we obtain a distribution based on leverage-scores. Using this distribution, one can construct the sketching matrix $\bSS_{s,m}$ by sampling (with replacement) rows of the $m \times m$ identity matrix $\eye_m$ according to the distribution \cref{eqn:sampling_dist}. Each selected row $i$ is scaled by a multiplicative factor 
\begin{align}
\label{rescaling-factor}
\frac{1}{\sqrt{s \pi_{m,n}(i)}},
\end{align}
to ensure 
\begin{align*}
\mathbb{E} \left[ \| \bSS_{s,m} \AA_{m,n} \xx \|^2 \right] = \| \AA_{m,n} \xx \|^2 \quad \forall \xx ,
\end{align*}
which guarantees unbiased solution estimates \cite{RandNLA_PCMIchapter_chapter}. In (\ref{rescaling-factor}), $s$ is the number of sampled rows. 

Obviously, obtaining the leverage-scores \emph{exactly} is computationally expensive, requiring the same amount of resources in solving the original $\mathtt{OLS}$ problem. 
However, they can be computed \emph{approximately} more quickly \cite{drineas2012fast}, basically in the time it takes to implement a random projection. Additionally, as highlighted in \cite{mahoneyRandomizedAlgorithmsMatrices2011,RandNLA_PCMIchapter_chapter}, even with inaccurately estimated leverage-scores by a factor $0 < \beta \leq 1$ expressed as
\begin{align*}
\hat{\ell}(i) \geq \beta \ell(i), \quad \mbox{for}\ i = 1,2,\ldots m,
\end{align*}
we can still achieve the subspace embedding property \cref{eqn:subspace_embedding} with a sample size $s$ as given by
\begin{align}
\label{eq:lev_beta_sample}
s \in \bigO{n \log(n/\delta)/(\beta \varepsilon^{2})}.
\end{align}
Equation \cref{eq:lev_beta_sample} shows that a larger sample size is required when $\beta$ is smaller, reflecting poorer estimates. Furthermore, if the goal is to reduce the error ($\varepsilon$), the sample size must be increased accordingly.

\subsection{Related Work}
Calculating leverage-scores for an arbitrary tall and thin data matrix $\AA_{m,n}$, when $m$ is significantly larger than $n$, through na\"{i}ve algorithms based on orthogonal bases for the matrix's column space, requires a huge computational load. This computational process needs $\bigO{m n^{2}}$ time, which not only proves to be computationally expensive but also contradicts the intention of using sampling techniques to reduce computational overhead. To address computational challenges in computing leverage-scores, numerous studies have been conducted to develop non-trivial and inventive methods for approximating them. One of the pioneering works to approximate leverage-scores can be found in \cite{drineas2012fast}. This approach uses a rapid Johnson-Lindenstrauss mapping technique to transform a collection of high-dimensional data points into a lower-dimensional space while conserving the pairwise distances among these points. Specifically, the mapping takes the shape of a subsampled randomized Hadamard transform and plays a central role in designing a randomized algorithm for approximating leverage-scores within $\varepsilon$-relative accuracy in $\bigO{m n \varepsilon^{-2}\log m}$ time. Subsequent research has further enhanced this bound, mostly with respect to input sparsity \cite{mengLowdistortionSubspaceEmbeddings2013,nelson2013osnap, clarkson2017low,sobczyk2021estimating}. In particular, the bound was improved to $\bigO{\varepsilon^{-2}\textnormal{nnz}(\AA_{m,n}) \log m}$, which is essentially the optimal bound attainable considering factors such as logarithmic terms \cite{sobczyk2021estimating}.

Iterative sampling algorithms aimed at approximating leverage-scores have attracted attention in this domain as well \cite{li2013iterative,cohen2015uniform}. In \cite{li2013iterative}, the authors adopt an iterative approach that alternates between the calculation of a short matrix estimate of $\AA_{m,n}$ and the computation of increasingly accurate approximations for the leverage-scores. This iterative process results in a runtime of $\bigO{\textnormal{nnz}(\AA_{m,n})+n^{\omega+\theta}\varepsilon^{-2}}$, where $\theta>0$ and $n^{\omega+\theta}$ is similar to the computational complexity of solving a regression problem on reduced matrices. Additionally, an iterative row sampling algorithm for matrix approximation is proposed in \cite{cohen2015uniform}, emphasizing the preservation of the row structure and sparsity in the original data matrix. This study focuses on a comprehensive examination of the information loss when employing uniform sampling in the context of solving regression problems. Their findings lead to the conclusion that by appropriately assigning weights to a small subset of rows, minimizing the coherence becomes feasible, represented by the largest leverage-score, for any arbitrary matrix. Furthermore, they demonstrate the possibility of achieving constant-factor approximations of leverage-scores, where $\ell(i) \leq \hat{\ell}(i) \leq n^{\theta} \ell(i)$ holds for $i = 1,2,\ldots, m$, and this approximation can be computed in $\bigO{\textnormal{nnz}(\AA_{m,n})/\theta}$ time.

We should note quantum-inspired methods for approximating leverage-scores have been studied as well \cite{liu2017fast,zuo2021quantum}. In \cite{liu2017fast}, the authors propose an efficient quantum algorithm suitable for least-squares regression problems involving sparse and well-conditioned matrix $\AA_{m,n}$. For matrices that are ill conditioned or do not possess full column rank, their approach integrates regularization techniques such as ridge regression and $\delta$-truncated singular value decomposition. This quantum-inspired technique enables the approximation of leverage-scores within $\varepsilon$-absolute error, in $\bigO{\textnormal{nnz}(\AA_{m,n})^{2} (m \varepsilon)^{-1} \log m}$ time. In another quantum-inspired approach for approximating leverage-scores, as introduced in \cite{zuo2021quantum}, the method relies upon certain natural sampling assumptions. Their technique involves the computation of singular values and their corresponding right singular vectors from a compact submatrix. The results of this step are used in approximating the left singular matrix, ultimately enabling the approximation of the leverage-scores.

Lastly, note a closely related line of research is dedicated to approximating an extension of leverage-scores called ridge leverage-scores in the context of kernel ridge regression. To effectively reduce the dimensionality of data matrices in kernel-based methods, such as kernel ridge regression using leverage-score-based distributions, Alaoui and Mahoney present a method that offers rapid approximations of ridge leverage-scores using random projections in \cite{alaoui2015fast}. Furthermore, Cohen et al. demonstrated in \cite{cohen2017input} that by using the monotonic property of ridge leverage-scores, approximating them using a relatively large uniform subsample of columns of matrix $\AA_{m,n}$ is possible. Additionally, ridge leverage-scores are used in the development of a fast recursive sampling scheme for kernel Nystr\"{o}m approximation \cite{musco2017recursive}. Other uses of leverage-scores when the data matrix is positive definite for kernel ridge regression are reported in \cite{rudi2018fast} and \cite{chen2021fast}.

	%-----------------------
	% Theoretical Results
	%-----------------------

	\section{$\mathtt{SALSA}$: Theoretical Results} \label{chap_theo}
In this section, we provide the theoretical foundation and employ tools from $\mathtt{RandNLA}$ to support the creation of a recursive algorithm aimed at approximating the leverage-scores associated with the rows of any arbitrary matrix. The algorithm is called $\mathtt{SALSA}$. More specifically, we start by introducing a recursive approach for the exact computation of leverage-scores of the matrix's rows. Then, we present a complementary approximate recursion that forms the foundation of $\mathtt{SALSA}$. It is theoretically proved that estimates from $\mathtt{SALSA}$ have desirable relative error bounds with high probability.

%---------------------------------------
%\subsection{Exact Leverage Score Recursion}
\subsection{A Recursive Procedure for Exact Leverage-Scores}
\label{sec:elsr}
%---------------------------------------

To explain the central theorem regarding a recursive approach for computing the exact leverage-scores of a data matrix, we describe a column-wise matrix augmentation procedure. This procedure starts with the first column and incrementally incorporates additional columns one by one until the entire matrix is reconstructed. 
This procedure is similar to the one devised in \cite{eshraghLSAREfficientLeverage2019}; however, it does not rely on any predefined structure for the original matrix. From the original matrix $\AA_{m,n}$, suppose we have already constructed the following submatrix:
\begin{equation}
\AA_{m,d} = \left(\begin{array}{ccccccccc}
\aa_0 & | & \aa_1 & |  & \cdots & | & \aa_{d-1} \\
\end{array} \right) , \,
%\vspace{-0.3cm}
\label{eqn:A_md_in_col_form}
\end{equation}
where $\aa_0 = \AA_{m,n}(:,1)\in \mathbb{R}^{m \times 1}$, and the first column of $\AA$, and $\aa_1, \ldots, \aa_{d-1}$ are the next $d-1$ columns in the same order they appear in the matrix. The next submatrix, $\AA_{m,d+1}$, is constructed from $\AA_{m,d}$ by adding the $d^{\text{th}}$ column of $\AA_{m,n}$ to the end of $\AA_{m,d}$:
\begin{equation}
\AA_{m,d+1} = \bigg( \AA_{m,d} \;\;  | \;\; \aa_{d} \bigg).
\label{eqn:define_A_recursively}
\end{equation}
Now, suppose $\bm{\phi}_{m,d} \in \mathbb{R}^{d \times 1}$ is the optimal solution for the OLS problem $\min\limits_{\bm{\phi}} \|\AA_{m,d} \bm{\phi} - \aa_d \|$, which is analytically described using the normal equation
\begin{equation}
\bm{\phi}_{m,d} = \left( \AA_{m,d}^\transpose \AA_{m,d} \right)^{-1} \AA_{m,d}^\transpose \aa_d.
\label{eqn:definition_phi}
\end{equation}
Accordingly, the associated residual vector $\rr_{m,d}$ is defined as
\begin{equation}
\rr_{m,d} = \AA_{m,d} \bm{\phi}_{m,d} - \aa_d . 
\label{eqn:exact_residual_defn}
\end{equation}

%-----------------------------------------------------------
With the provided matrix augmentation procedure, the following theorem describes the recursive computation of leverage-scores of an augmented matrix using the leverage-scores of its preceding submatrix.
\begin{theorem}[Recursive Scheme for Exact Leverage-Scores]
	Consider the matrix $\AA_{m,d-1}$ as in~\cref{eqn:A_md_in_col_form} and~\cref{eqn:define_A_recursively} $\left( m>d \right)$ and its corresponding leverage-scores $\ell_{m,d-1} (i)$ for $i=1,...,m$. Then the leverage-scores of the augmented matrix $\AA_{m,d}$ can be computed as
	\begin{equation*}
	\ell_{m,d}(i) = \ell_{m,d-1}(i) + \frac{(\rr_{m,d-1}(i))^2}{\| \rr_{m,d-1}\|^2} ,
	%\label{eqn:ELR_Thm1}
	\end{equation*}
	$\text{for } d \geq 1, \text{ and } i=1,\ldots,m,$ where $\rr_{m,d}$ is defined in~\cref{eqn:exact_residual_defn} and $\ell_{m,1}(i) = \frac{(\aa_0 (i))^2}{\| \aa_0 \|^2} \quad \text{for } i=1,\ldots,m$.
	\label{thm:General_Lev_Score_Recursion}
\end{theorem}
The recursion depicted in Theorem \ref{thm:General_Lev_Score_Recursion} computes the leverage-scores of $\AA_{m,d}$ as the summation of the leverage-scores of $\AA_{m,d-1}$ and those for the associated residual vector $\rr_{m,d-1}$. Also, a leverage-score-based sampling distribution for every $\AA_{m,d}$ can be constructed as follows:
\begin{equation}
    \pi_{m,d} (i) := \frac{\ell_{m,d} (i)}{d} \quad \text{for } i=1,\ldots,m.
    \label{eqn:Def1_prob_dist}
\end{equation}

%---------------------------------------
%\subsection{Approximate Leverage Score Recursion}
\subsection{A Recursive Procedure for Approximate Leverage-Score}
\label{sec:alsr}
%---------------------------------------
A direct application of Theorem \ref{thm:General_Lev_Score_Recursion} proposes an elegant but computationally expensive method for computing a matrix's leverage-scores. However, employing this recursion involves sequentially solving multiple $\mathtt{OLS}$ problems, demanding $\bigO{mn^3}$ flops, in contrast to the direct computation via the hat matrix, which requires $\bigO{mn^2}$ flops. To address this computational discrepancy, we introduce matrix sketching into the recursive scheme. This enhancement involves the use of two distinct instances of sketching, as follows:

\begin{itemize}
    \item [\textbf{(a)}] when solving the $\mathtt{OLS}$ problem to compute $\bm{\phi}_{m,d}$ as defined in (\ref{eqn:definition_phi}); and
	%\vspace{-0.3cm}
	\item [\textbf{(b)}] when performing matrix multiplication to calculate $\rr_{m,d}$ as defined in  (\ref{eqn:exact_residual_defn}).
\end{itemize}

For further clarification, the first application of sketching in \textbf{(a)} involves reducing the number of rows in the matrix $\AA_{m,d}$ through a sampling process to produce $\hat{\AA}_{m,d}:= \bSS^{(1)}_{s_1,m} \AA_{m,d}$. Here, $\bSS^{(1)}_{s_1,m} \in \mathbb{R}^{s_1 \times m}$ represents a sampling matrix with $s_1$ rows randomly chosen with replacement from the rows of the $m \times m$ identity matrix $\eye_m$. The selection of rows follows the distribution  $\{\pi_{m,d}(i)\}_{i=1}^m$ outlined in  \cref{eqn:Def1_prob_dist}, rescaled by the appropriate factor, as described in \cref{rescaling-factor}. In Definition~\ref{def:Defn_Two}, \cref{thm:rel_err} replaces this distribution with one that incorporates approximate leverage-scores. An estimation of $\bm{\phi}_{m,d}$ is represented by $\hat{\bm{\phi}}_{m,d}$, which is defined as

\begin{equation}
\hat{\bm{\phi}}_{m,d} := (\hat{\AA}_{m,d}^\transpose \hat{\AA}_{m,d})^{-1} \hat{\AA}_{m,d}^\transpose \hat{\aa}_d,
\label{eqn:phi_hat_defined}
\end{equation}
where $\hat{\aa}_d := \bSS^{(1)}_{s_1,m} \aa_d$. With this approximation, the approximate residual vector is given by
\begin{equation}
\hat{\rr}_{m,d} := \AA_{m,d} \hat{\bm{\phi}}_{m,d} - \aa_{d}.
\label{eqn:r_hat_defined}
\end{equation}

Furthermore, for larger values of $d$, an additional approximation can be applied to $\AA_{m,d} \hat{\bm{\phi}}_{m,d}$ to achieve better computational efficiency. This approximation can be accomplished by employing the BasicMatrixMultiplication algorithm (of~\cite{drineasFastMonteCarlo2006}), which forms the foundation for the second application of sketching mentioned in \textbf{(b)}.

Following this algorithm, the number of columns in the matrix $\AA_{m,d}$ is further reduced through a sampling process, resulting in $\hat{\hat{\AA}}_{m,d}:= \AA_{m,d} \bSS^{(2)}_{d,s_2}$. In this context, $\bSS^{(2)}_{d,s_2} \in \mathbb{R}^{d \times s_2}$ represents the sampling matrix, with $s_2$ columns randomly selected with replacement from the columns of the $\eye_d$ matrix. The selection of columns follows the distribution defined by the leverage-scores of $\hat{\bm{\phi}}_{m,d}$. Subsequently, a new estimate for the residual vector is obtained through
\begin{equation}
\hat{\hat{\rr}}_{m,d} := \hat{\hat{\AA}}_{m,d} \hat{\hat{\bm{\phi}}}_{m,d} - \aa_d,
\label{eqn:r_hat_hat_defined}
\end{equation}
where 
\begin{equation}
\hat{\hat{\bm{\phi}}}_{m,d} := \bSS^{(2)\transpose}_{d,s_2} \hat{\bm{\phi}}_{m,d}.
\label{eqn:phi_hat_hat_defined}
\end{equation}
Note that when dealing with values of $d < s_2$, the entire matrix $\AA_{m,d}$ is used without any sampling. In other words, $\bSS^{(2)}_{d,s_2}  = \eye_d$ for these cases. However, when $d \geq s_2$, we introduce the sketching process as described earlier.

The approximate algorithm constructed from this procedure, as outlined in Algorithm~\ref{alg:salsa}, demands $\bigO{mn(s_1^2+s_2)}$ flops. Furthermore, the relative error associated with this algorithm can be favorably bounded a priori. In what follows, we introduce key theorems that establish this relative error bound and provide a comprehensive definition of the algorithm and its components.

The following theorem provides bounds for the estimates presented in  \cref{eqn:phi_hat_defined} and \cref{eqn:r_hat_defined}.

\begin{theorem}[\cite{drineasFasterLeastSquares2011}]
	Let $0<\varepsilon_1$, $\delta<1$. For sampling with (approximate) leverage-scores, using a sample size $s$ as described in (\ref{eq:lev_beta_sample}), the following inequalities hold with a probability of at least $1-\delta$:
	\begin{equation*}
	\| \rr_{m,d} \| \leq \| \hat{\rr}_{m,d} \| \leq (1+\varepsilon_1)\| \rr_{m,d} \|,
	\label{eqn:bound_on_lev_score_sampling_OLS_residuals}
	\end{equation*}
	\begin{equation*}
	\| \bm{\phi}_{m,d} - \hat{\bm{\phi}}_{m,d} \| \leq \sqrt{\varepsilon_1} \eta_{m,d} \| \bm{\phi}_{m,d} \|,
	\end{equation*}
	where $\bm{\phi}_{m,d}$, $\rr_{m,d}$, $\hat{\bm{\phi}}_{m,d}$, and $\hat{\rr}_{m,d}$ are defined respectively in  \cref{eqn:definition_phi}, \cref{eqn:exact_residual_defn}, \cref{eqn:phi_hat_defined}, and \cref{eqn:r_hat_defined}. Additionally,
	\begin{equation*}
	\eta_{m,d} = \kappa ( \AA_{m,d} ) \sqrt{\zeta^{-2} - 1},
	\end{equation*}
	where $\kappa( \AA_{m,d} )$ represents the condition number of matrix $\AA_{m,d}$, and $\zeta \in (0,1]$ denotes the fraction of $\aa_d$ lying within \text{Range}$(\AA_{m,d})$. In other words, $\zeta$ is defined as $\zeta := \| \HH_{m,m} \aa_d \| / \| \aa_d \|$ with $\HH_{m,m}$ is the hat matrix corresponding to $\AA_{m,d}$.
	\label{thm:michael}
\end{theorem}

Next, we present \cref{thm:matrix_mult_error_bound}, which appears as a consequence of Theorem \ref{thm:General_Lev_Score_Recursion} in~\cite{drineasFastMonteCarlo2006}. This theorem provides bounds for the estimates produced by the BasicMatrixMultiplication algorithm (of~\cite{drineasFastMonteCarlo2006}). This algorithm is a sampling-based approach aimed at approximating the product of two large matrices while preserving the dimensionality of the resulting product. This randomized matrix multiplication algorithm plays an important role in various other $\mathtt{RandNLA}$ algorithms. Let's consider two matrices, $\AA_{m,n} \in \mathbb{R}^{m \times n}$ and $\BB_{n,p} \in \mathbb{R}^{n \times p}$. The multiplication between them is defined as
$$\AA_{m,n}\BB_{n,p}=\sum\limits_{t=1}^n\AA_{m,n}(:,t)\BB_{n,p}(t,:).$$
This summation involves the addition of $n$ rank-one matrices, each of which results from the outer product of a column from matrix $\AA_{m,n}$ and a row from matrix $\BB_{n,p}$. These rank-one matrices have the same dimensionality of $m\times p$ as the product matrix. The outer product characteristic creates an opportunity to approximate this product using a reduced number of rank-one matrices. This reduction is achieved by carefully selecting (with replacement) corresponding columns and rows from matrices $\AA_{m,n}$ and $\BB_{n,p}$ based on a given probability distribution, scaled as follows:
$$\sum\limits_{t=1}^c\frac{1}{c\pi_{i_t}}\AA_{m,n}(:,i_t)\BB_{n,p}(i_t,:).$$
Here, $i_t \in \{1, \ldots, n\}$ is chosen such that $\Pr(i_t = k) = \pi_k$. This approach allows for efficient matrix multiplication approximations while reducing computational burden.

\begin{theorem}[\cite{drineasFastMonteCarlo2006}]
    Consider matrices $\AA_{m,n} \in \mathbb{R}^{m \times n}$ and $\BB_{n,p} \in \mathbb{R}^{n \times p}$, and a positive integer $c$ such that $1 \leq c \leq n$. Let $\{\pi_i\}_{i=1}^n$ be a set of non-negative scalars satisfying  $\sum\limits_{i=1}^n \pi_i = 1$, and for a positive constant $\beta \leq 1$, let them also satisfy the condition 
    $$\pi_i \geq \dfrac{\beta \| \AA_{m,n}(:,i) \| \| \BB_{n,p}(i,:) \|}{\sum_{k=1}^n \| \AA_{m,n}(:,k) \| \| \BB_{n,p}(k,:) \|}.$$ 
    We define $\CC_{m,c} \RR_{c,p}$ as the approximation to the matrix product $\AA_{m,n} \BB_{n,p}$ obtained using the BasicMatrixMultiplication algorithm (of~\cite{drineasFastMonteCarlo2006}). If $c$ satisfies the condition $c \geq {\xi^2}/{(\beta \varepsilon_2^2)}$, we can establish, with a probability of at least $1-\delta$, the following inequality:
	\begin{equation*}
	\| \AA_{m,n} \BB_{n,p} - \CC_{m,c} \RR_{c,p} \|_F \leq \varepsilon_2 \| \AA_{m,n} \|_F \| \BB_{n,p} \|_F,
	%\label{eqn:matrix_mult_bound}
	\end{equation*}
	where $\xi = 1+\sqrt{(8/\beta)\log (1/\delta)}$.
	\label{thm:matrix_mult_error_bound}
\end{theorem}

\cref{thm:matrix_mult_error_bound} shows that by sampling a sufficient number of column-row pairs, the approximation error $\| \AA_{m,n} \BB_{n,p} - \CC_{m,c} \RR_{c,p} \|_F$ can be reduced to arbitrary precision. Now, we proceed to introduce the recursive procedure for approximating the leverage-scores in Definition~\ref{def:Defn_Two}. This recursive procedure is analogous to the exact one presented in Theorem~\ref{thm:General_Lev_Score_Recursion}, using the two instances of matrix sketching previously described.

%-----------------------------------------------------------

\begin{definition}[Sequential Approximate Leverage-Scores]\label{def:seq_ls}
	%\begin{subequations}
	Consider matrix $\AA_{m,d}$ defined recursively as $\AA_{m,d+1} = \bigg( \AA_{m,d} \, | \, \aa_d \bigg)$. The sequential approximate leverage-scores are defined by the following recursion:
	\begin{equation*}
	\hat{\ell}_{m,d}(i) := \begin{cases}
	\ell_{m,1}(i),  &\text{for } d=1, \\
	%\vspace{0.1cm}
	\hat{\ell}_{m,d-1}(i) + \dfrac{(\hat{\rr}_{m,d-1}(i))^2}{\| \hat{\rr}_{m,d-1} \|^2},  &\text{for } 2 \leq d \leq s_2, \\
	%\vspace{0.3cm}
	\hat{\ell}_{m,d-1}(i) + \dfrac{(\hat{\hat{\rr}}_{m,d-1}(i))^2}{\| \hat{\hat{\rr}}_{m,d-1} \|^2},  &\text{for } d > s_2, \\
	\end{cases}
	%\label{eqn:defn_approx_lev_scores}
	\end{equation*}
	where $\hat{\rr}_{m,d-1}$ and $\hat{\hat{\rr}}_{m,d-1}$ are given in  \cref{eqn:r_hat_defined,eqn:r_hat_hat_defined}, respectively. 
	\label{def:Defn_Two}
\end{definition}

We now introduce $\mathtt{SALSA}$ as outlined in Algorithm~\ref{alg:salsa}. This algorithm is designed to provide approximate leverage-scores for any matrix $\AA_{m,n} \in \mathbb{R}^{m \times n}$. \cref{thm:rel_err} describes the relative error associated with the leverage-score estimates produced by this algorithm. 

% %-----------------------------------------------------------
\begin{algorithm}[h!]
	\caption{\texttt{SALSA}: Sequential Approximate Leverage-Scores Algorithm}	
	\begin{algorithmic}
		\STATE \textbf{Input:} 
		\begin{itemize}[label=-]
			\vspace{2mm}
			\item Matrix $\AA_{m,d} = \left(\begin{array}{ccccccccc}
                            \aa_0 & | & \aa_1 & |  & \cdots & | & \aa_{d-1} \\
                            \end{array} \right) \in \mathbb{R}^{m \times n}$ as in \cref{eqn:A_md_in_col_form};
			\vspace{2mm}
		\end{itemize}
		\vspace{2mm}
		\STATE \emph{Step -1}. Set \( s_2 \) as in \cref{thm:matrix_mult_error_bound}, replacing \( c \) with \( s_2 \).; 
		\vspace{2mm}
		\item \emph{Step 0}. Set $d=1$ and $\AA_{m,d} = \aa_{d-1} $;
		\vspace{2mm}
		\WHILE {$ d < n $} 
		\vspace{2mm}
		\STATE \emph{Step 1}. Compute the approximate leverage-scores $\hat{\ell}_{m,d}(i)$ for $i=1,\ldots,m$ as in \cref{def:seq_ls};
		\vspace{2mm}
		\STATE \emph{Step 2}. Compute the sampling distribution $\hat{\pi}_{m,d}(i)$ for $i=1,\ldots,m$ as in \cref{eqn:Def1_prob_dist};
		\vspace{2mm}
		\STATE \emph{Step 3}. Set $s_1$ as in \cref{eq:lev_beta_sample} by replacing $n$ with $d$; %and $\beta$ with the bound given in ??; 
		\vspace{2mm}
		\STATE \emph{Step 4}. Form the $s_1 \times m$ sampling matrix $\bSS_{s_1,m}^{(1)}$ by randomly choosing $s_1$ rows of the corresponding identity matrix according to the probability distribution found in Step $3$, with replacement, and rescaling them with the factor \cref{rescaling-factor}\,;
		\vspace{2mm}
		\STATE \emph{Step 5}. Construct the sampled matrix $\hat{\AA}_{m,d} = \bSS_{s,m}^{(1)} \AA_{m,d}$ and response vector $\hat{\aa}_{d} = \bSS_{m,d} \aa_{d}$; 
		\vspace{2mm}
		\STATE \emph{Step 6}. Solve the associated reduced OLS problem to find the parameters $\hat{\bm{\phi}}_{m,h}$ as in \cref{eqn:phi_hat_defined};
		\vspace{2mm}
		\STATE \emph{Step 7}. If $d \leq s_2$, set $\bSS_{d,s_2}^{(2)} = I_{d,d}$ and go to Step $10$;
		\vspace{2mm}
		\STATE \emph{Step 8}. Compute the sampling distribution $\hat{\hat{\pi}}_{m,d}(i)$ for $i=1,\ldots,d$ as in \cref{thm:rel_err};
		\vspace{2mm}
		\STATE \emph{Step 9}. Form the $d \times s_2$ sampling matrix $\bSS_{d,s_2}^{(2)}$ by randomly choosing $s_2$ columns of the corresponding identity matrix according to the probability distribution found in Step $7$, with replacement, and rescaling them with the factor \cref{rescaling-factor}\,;
		\vspace{2mm}
		\STATE \emph{Step 10}. Construct the sampled matrix $\hat{\hat{\AA}}_{m,d}= \AA_{m,d} \bSS_{d,s_2}^{(2)} $ and estimated parameters $\hat{\hat{\phi}}_{m,d} =  \bSS_{d,s_2}^{(2) \transpose} \hat{\phi}_{m,d} $; 
		\vspace{2mm}
		\STATE \emph{Step 11}. Calculate the residuals $\hat{\hat{\res}}_{m,d}$ as in \cref{eqn:phi_hat_hat_defined};
		\vspace{2mm}
		\STATE \emph{Step 12}. Set $\AA_{m,d+1} = \bigg( \AA_{m,d} \, | \, \aa_d \bigg)$ and $d \leftarrow d + 1$;
		\vspace{2mm}
		\ENDWHILE
		\vspace{2mm}
		\STATE \emph{Step 13}. Compute the approximate leverage-scores $\hat{\ell}_{m,d}(i)$ for $i=1,\ldots,m$ as in \cref{def:seq_ls};
		\vspace{2mm}
		\STATE \textbf{Output:} Estimated leverage-scores $ \hat{\ell}_{m,n}(i) $ for $i=1,\ldots,m$.
	\end{algorithmic}
	\label{alg:salsa}
\end{algorithm}
%-----------------------------------------------------------

%-----------------------------------------------------------
\begin{theorem}[Relative Errors for Sequential Approximate Leverage-Scores]
	%For the approximate leverage scores \cref{eqn:defn_approx_lev_scores}, if the row-sketch of size $s_1$ is done based on the leverage score sampling distribution
    When computing the approximate leverage-scores, as defined in  \cref{def:Defn_Two}, if the row-sketch of size $s_1$ is performed based on the leverage-score sampling distribution,
	\begin{equation*}
	\hat{\pi}_{m,d} (i) = \frac{\hat{\ell}_{m,d}(i)}{d} \quad \text{for } i=1,\ldots,m,
	%\label{eqn:approx_lev_score_dist}
	\end{equation*}
	and the column-sketch of size $s_2$ is performed based on the leverage-score distribution
	\begin{equation*}
		\hat{\hat{\pi}}_{m,d}(i) = \frac{(\hat{\bm{\phi}}_{m,d}(i))^2}{\|\hat{\bm{\phi}}_{m,d}\|^2}\quad \text{for } i=1,\ldots,d,
		%\label{eqn:phi_lev_score_dist}
	\end{equation*}
    we can establish, with a probability of at least $1-\delta$, the following inequality:
	%we have with probability at least $1-\delta$
	\begin{align*}
		\frac{| \ell_{m,d}(i) - \hat{\ell}_{m,d}(i) |}{\ell_{m,d}(i)} & \leq \big( \sqrt{\xi_{m,d-1}} + 5(\eta_{m,d-1}+2) \kappa^2 (\AA_{m,d}) \big) (d-1) \sqrt{\varepsilon^\star},
	\end{align*}
	for $i=1,\ldots,m$, where 
	\begin{align*}
		\xi_{m,d-1} & = 1+2\kappa(\AA_{m,d-1}) \sqrt{d-1},
	\end{align*}	
	and $\eta_{m,d}$ and $\kappa(\AA_{m,d})$ are as in Theorem~\ref{thm:michael} and $\varepsilon^\star = \max\{\varepsilon_1,\varepsilon_2\}$. Note $\varepsilon_1$ is defined in Theorem~\ref{thm:michael} and $\varepsilon_2$ is defined in Theorem~\ref{thm:matrix_mult_error_bound}.
	\label{thm:rel_err}
\end{theorem}
    	
	%-----------------------
	% SALSA_Numerical Results
	%-----------------------
	
	\section{$\mathtt{SALSA}$: Numerical Results}
\label{sec:SALSA_numerical_results}

In this section, we present the numerical results that demonstrate the efficiency of $\mathtt{SALSA}$ in estimating leverage-scores. The numerical experiments were conducted using MATLAB R2022a on a Windows machine equipped with an Intel(R) Xeon(R) CPU @3.60GHz 6 cores. To ensure fair comparisons and consistency, all computations are performed using a single thread in MATLAB, despite the multi-core capacity of the device for parallel computing.  

\subsection{Synthetic Big Data}
To compare the performances of exact leverage-score calculations using (\ref{HatMatrix}) with the estimation method implemented in $\mathtt{SALSA}$, we conducted numerical experiments on a synthetic matrix $\AA_{m\times n}$, with $m=20,000,000$ rows and $n=300$ columns. The entries of this matrix were generated from a standard Gaussian distribution. To introduce outliers, we randomly selected $2,000$ rows and modified them. We adjusted these outlier rows by adding random numbers generated from a $t$-distribution with one degree of freedom, multiplied by a scalar to make them different from the original (normal) rows. The inclusion of these outliers allows us to assess the algorithm's capability to handle non-uniformity effectively, as expected. The distribution of leverage-scores of one of these data matrices is depicted in Figure \ref{fig:ls_syn}, which shows the non-uniformity included in the synthetic data matrix.

 \begin{figure*}[!ht]
    \centering
    \includegraphics[width=0.5\textwidth]{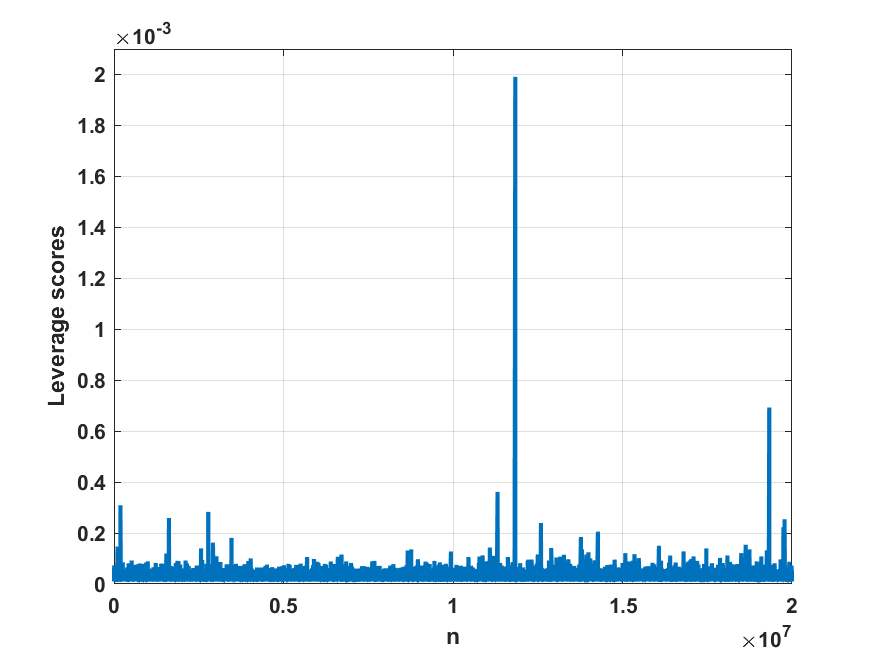}
    \caption{Non-uniform distribution of leverage-scores of the synthetic matrix with $m=20,000,000$ rows and $n=300$ columns, deliberately altered to include $2,000$ outliers.}   
    \label{fig:ls_syn}
\end{figure*}

Following Algorithm \ref{alg:salsa} for estimating leverage-scores using $\mathtt{SALSA}$, one needs to specify $s_1$, the row-sketch size for the $\mathtt{OLS}$ step, and $s_2$, the column-sketch size for the matrix/vector multiplication involved in calculating the residual vector. Hence, we investigate the impact of different choices of $s_1$ and $s_2$ on the execution time in estimating the leverage-scores and on the amount of deviation from the exact leverage-scores. In this experiment, we consider $s_1 \in \{0.005n, 0.001n, 0.002n, 0.004n\}$ and $s_2 \in \{1,2,4,8\}$. For each choice of $s_1$ and $s_2$, Figure \ref{fig:SALSA_time_error} records and depicts the average time and mean absolute percentage error (MAPE) of fifty independent runs. The $\mathtt{MAPE}$ between exact and estimated leverage-scores is calculated as follows:
\begin{equation}
  \mathtt{MAPE} =  \frac{1}{m}\sum\limits_{i=1}^m \left |\frac{\bm{\ell}_{m,n}(i) - \hat{\bm{\ell}}_{m,n}(i)}{\bm{\ell}_{m,n}(i)} \right | \times 100\%.
  \label{percentageError}
\end{equation}

% \begin{figure}[!ht]
% \vskip 0.2in
% \begin{center}
% \centerline{
%     \includegraphics[width=0.5\columnwidth]{time_s1.png}
%     \includegraphics[width=0.5\columnwidth]{time_s2.png}}
% \centerline{
%     \includegraphics[width=0.5\columnwidth]{error_s1.png}
%     \includegraphics[width=0.5\columnwidth]{error_s2.png}}
% \caption{Run times against $s_1$ (top left) and $s_2$ (top right) in computing exact and estimated leverage scores and percentage mean element-wise relative error  against $s_1$ (bottom left) and $s_2$ (bottom right) for a synthetic matrix with 20,000,000 rows, n=300 column, and altered to have 2,000 outliers.}
% \label{fig:SALSA_run_times}
% \end{center}
% \vskip -0.2in
% \end{figure}

\begin{figure*}[!ht]
        \centering
        \begin{subfigure}[b]{0.475\textwidth}
            \centering
            \includegraphics[width=\textwidth]{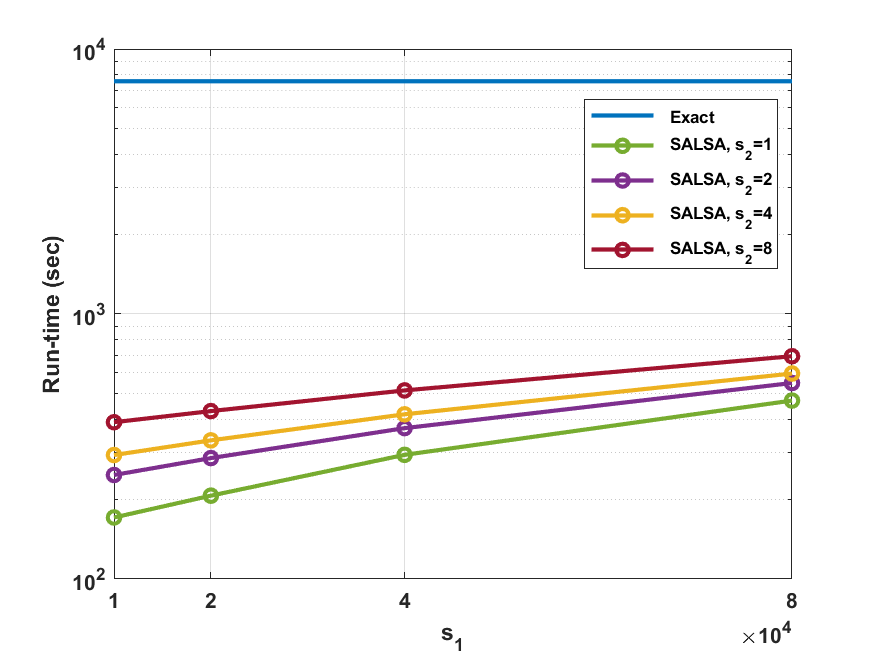}
            \caption[]{}    
            \label{fig:SALSA_time_error_a}
        \end{subfigure}
        \hfill
        \begin{subfigure}[b]{0.475\textwidth}  
            \centering 
            \includegraphics[width=\textwidth]{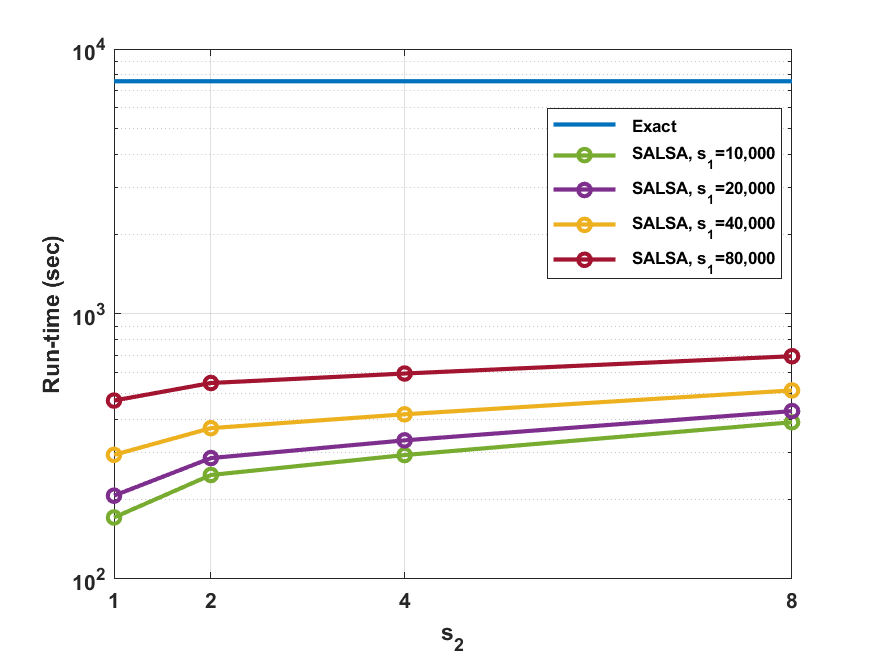}
            \caption[]{}    
            \label{fig:SALSA_time_error_b}
        \end{subfigure}
        \vskip\baselineskip
        \begin{subfigure}[b]{0.475\textwidth}   
            \centering 
            \includegraphics[width=\textwidth]{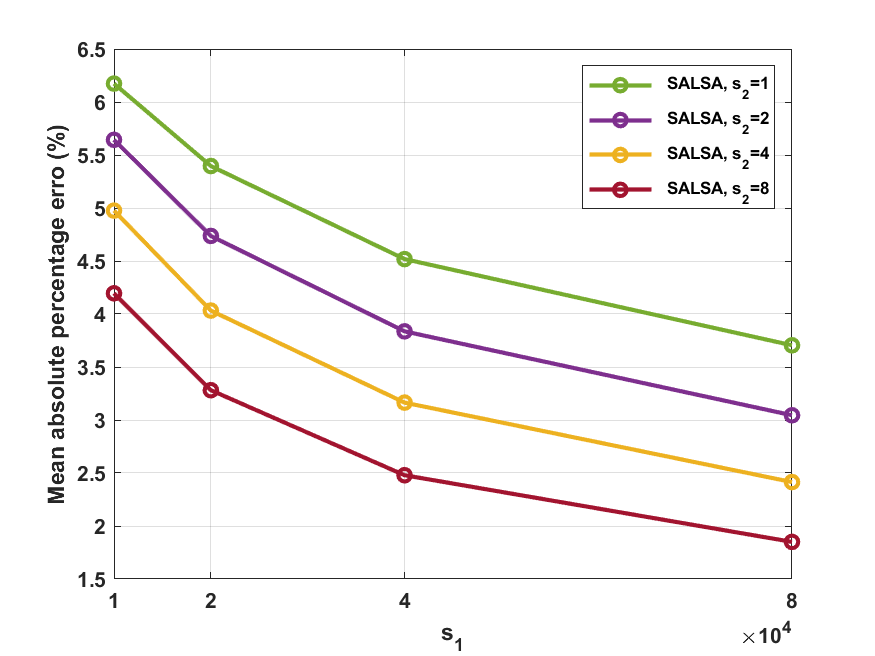}
            \caption[]{}    
            \label{fig:SALSA_time_error_c}
        \end{subfigure}
        \hfill
        \begin{subfigure}[b]{0.475\textwidth}   
            \centering 
            \includegraphics[width=\textwidth]{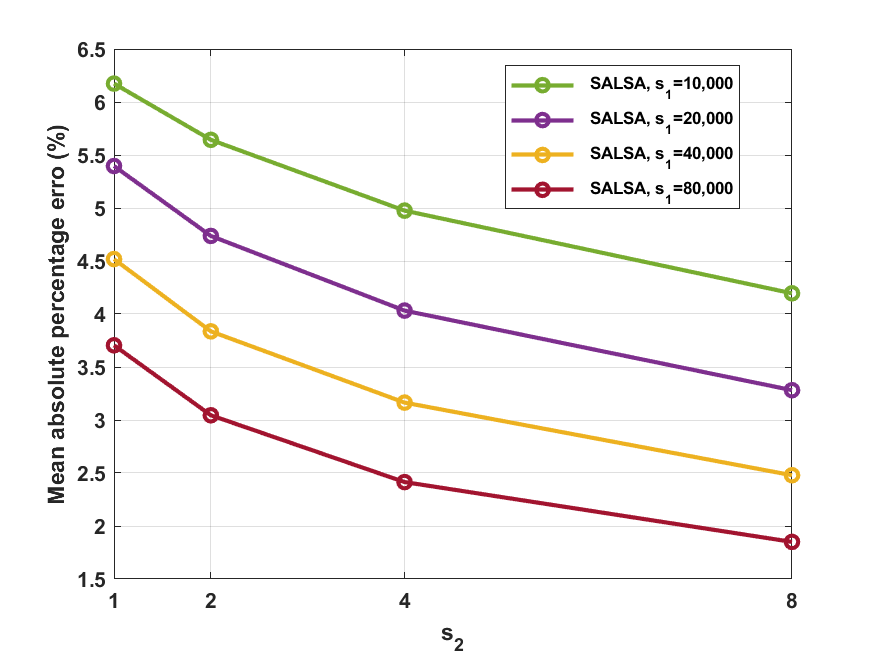}
            \caption[]{}    
            \label{fig:SALSA_time_error_d}
        \end{subfigure}
        \caption[]{(\ref{fig:SALSA_time_error_a}) shows the run times as a function of $s_1$ for different values of $s_2$ for computing both exact and estimated leverage-scores, while (\ref{fig:SALSA_time_error_b}) displays the run times as a function of $s_2$ for different fixed values of $s_1$. These results are obtained for a synthetic matrix with $m=20,000,000$ rows and $n=300$ columns, deliberately altered to include $2,000$ outliers. Furthermore, (\ref{fig:SALSA_time_error_c}) presents the $\mathtt{MAPE}$ against $s_1$, and (\ref{fig:SALSA_time_error_d}) depicts the same error metric against $s_2$ between the exact and estimated leverage-scores. } 
        \label{fig:SALSA_time_error}
\end{figure*}

Employing exact methods, such as (\ref{HatMatrix}), for computing leverage-scores demands substantial CPU time and RAM capacity, posing a significant obstacle for numerical experiments on standard computers. The computational cost of the exact method (\ref{HatMatrix}) is evident in Figures \ref{fig:SALSA_time_error_a} and \ref{fig:SALSA_time_error_b} when compared with the $\mathtt{SALSA}$ estimation method presented in this work. In our experiments with synthetic data, finding exact leverage-scores using (\ref{HatMatrix}) with single-threaded computations in MATLAB takes approximately $7,557$ seconds, equivalent to nearly two hours. Furthermore, a matrix with $20$ million rows and $300$ columns requires 45 GB of memory for storage. Using the most efficient approach for computing exact leverage-scores, by using $\mathtt{QR}$ decomposition, necessitates an additional 95 GB of memory. Consequently, the exact method demands a total of 140 GB of memory to successfully execute. However, by employing $\mathtt{SALSA}$, as described in Algorithm \ref{alg:salsa}, an extensive amount of RAM is not necessary. The error bounds demonstrate the estimated and exact leverage-scores are remarkably similar, with a $\mathtt{MAPE}$ consistently less than $6\%$ in the worst-case scenario as demonstrated in Figures \ref{fig:SALSA_time_error_c} and \ref{fig:SALSA_time_error_d}. This finding substantiates the practicality and accuracy of $\mathtt{SALSA}$ in estimating leverage-scores while circumventing the computational challenges associated with exact methods.

%In computing exact leverage scores of a data matrix of size $20,000,000 \times 300$, we need 45GB of memory to store it, 50GB of space to store the QR decomposition, and 45 GB of space for storing $QQ^T$ which indicates that finding exact leverage scores is not possible for ordinary computers. However, in using $\mathtt{SALSA}$ to estimate the leverage scores, we only need 45GB for storing the data matrix, and the amount of memory to run the recursive computations is negligible. 

Regarding the execution-time comparison, Figure \ref{fig:SALSA_time_error_a} clearly illustrate that as $s_1$ increases, regardless of the value of $s_2$, $\mathtt{SALSA}$ needs to solve larger $\mathtt{OLS}$ problems, resulting in increased computational time. However, even in the worst-case scenario, the algorithm only requires $692$ seconds, which is a remarkable eleven times faster than the exact method. Moreover, in the average case, $\mathtt{SALSA}$ completes in just $392$ seconds, making it $19$ times quicker than the time required for the exact method to compute the leverage-scores. This substantial reduction in computational time showcases the efficiency of $\mathtt{SALSA}$ in handling large-scale problems. Furthermore, similar behavior is observed for changes in $s_2$ in Figure \ref{fig:SALSA_time_error_b}. For each fixed value of $s_1$, as the value of $s_2$ increases, the execution time for $\mathtt{SALSA}$ increases only slightly. Figures \ref{fig:SALSA_time_error_a} and \ref{fig:SALSA_time_error_b} indicate that when $s_1$ and $s_2$ reach their highest values, as expected, $\mathtt{SALSA}$ needs to solve larger $\mathtt{OLS}$ problems, resulting in slightly longer execution times. Nevertheless, the differences in execution times between the highest and lowest values of $s_1$ and $s_2$ are negligible, further validating the consistent performance of $\mathtt{SALSA}$ across various scenarios.

The $\mathtt{SALSA}$ not only estimates leverage-scores in a remarkably short time, but also maintains low error bounds on the estimations. By selecting appropriate values for $s_1$ and $s_2$, $\mathtt{SALSA}$ efficiently extracts manageable-sized $\mathtt{OLS}$ problems from the original problem, which is solved using the exact algorithm. Despite some information being neglected in the sketching process, Figures \ref{fig:SALSA_time_error_c} and \ref{fig:SALSA_time_error_d} clearly demonstrate the mean percentage deviation of the estimated leverage-scores from the exact leverage-scores remains below $6\%$. Furthermore, in Figure \ref{fig:SALSA_time_error_c}, with a fixed value of $s_2$, as $s_1$ increases and $\mathtt{SALSA}$ incorporates more data, the deviation error decreases, indicating the algorithm's ability to produce more accurate results as it processes additional information. Similar behavior is observed in Figure \ref{fig:SALSA_time_error_d} when $s_2$ increases with a fixed value of $s_1$. Even when $s_1$ and $s_2$ are at their highest values, resulting in larger $\mathtt{OLS}$ problems, the percentage deviations remain at the lowest levels. Additionally, the difference between the $\mathtt{MAPE}$ when $s_1$ and $s_2$ are at their highest and lowest values is only about $3\%$, further affirming the consistency and reliability of $\mathtt{SALSA}$ across different scenarios. 

Figure \ref{fig:SALSA_time_error} provides clear evidence of the remarkable efficiency of the $\mathtt{SALSA}$ algorithm in estimating leverage-scores. As $s_1$ and $s_2$ increase, the algorithm acquires more data and solves progressively larger problems, resulting in longer execution times but lower percentage deviations. Notably, the figure indicates the differences in execution times and deviation errors between assigning the highest or lowest values to $s_1$ and $s_2$ are not substantial. This observation highlights a trade-off between two competing objectives: the level of deviation from the exact leverage-scores and the execution time. 

Setting $s_1=0.002n$ and $s_2=4$ indeed establishes a balance point where both the percentage deviation from the exact solution and the execution time reach a moderate level. This configuration proposes a well-balanced compromise, achieving reasonably accurate leverage-score estimates while keeping the computational overhead at a manageable size. In particular, with this setting, estimating the leverage-scores of a data matrix with dimensions of $20M \times 300$ takes only $293$ seconds (about five minutes), and the $\mathtt{MAPE}$ deviation is at most $5\%$. This remarkable result stands in contrast to the alternative of $7,557$ seconds (nearly two hours) required for exact calculations. Additionally, the advantage of not requiring a huge amount of RAM capacity further enhances the practicality and efficiency of $\mathtt{SALSA}$ in numerical experiments. 

$\mathtt{SALSA}$ introduces a new level of practicality for time series analysis, enabling it to be implemented on small devices such as smartphones and wearable technologies. This advancement opens up possibilities for performing health data or market data analysis directly on these devices. With the capability of fitting good models to the data and making trustworthy predictions, $\mathtt{SALSA}$ empowers users to engage in data-driven decision-making and gain valuable insights without relying on resource-intensive computing systems. This integration of time series analysis on portable devices enhances convenience, accessibility, and efficiency for a wide range of applications, contributing to improved health monitoring, personalized recommendations, and smarter financial strategies.

\subsection{Real-World Big Data}
The real-data matrix used in our study is derived from a data matrix associated with a large knowledge graph used for hypothesis generation in the field of medical studies. In the medical and biomedical sciences, the ability to uncover implicit relationships among various concepts for rapid hypothesis generation is important. Sybrandt et al.'s seminal work \cite{sybrandt2017moliere} made a remarkable contribution to this domain. They assembled an extensive dataset comprising more than 24.5 million documents, encompassing scientific papers, keywords, genes, proteins, diseases, and diagnoses. Their work involved the development of a method for constructing a comprehensive network of public knowledge and implementing a query process capable of revealing human-readable relationships between nodes in this network. For this purpose, they constructed a substantial graph from these objects, with 30,239,687 nodes, 4,023,336 explicitly specified zero edges, and 3,106,164 duplicate edges. The edge weights within this graph were recorded in a square data matrix, containing 30,239,687 rows and columns, with approximately 6,669,254,694 nonzero elements. We extracted a data matrix of size $m=20,000,000$ rows and $m=160$ from the weight matrix of the knowledge graph. The distribution of leverage-scores of this matrix is depicted in Figure \ref{fig:ls_real}, which shows the non-uniformity in the data matrix. 
 \begin{figure*}[!ht] 
            \centering
            \includegraphics[width=0.5\textwidth]{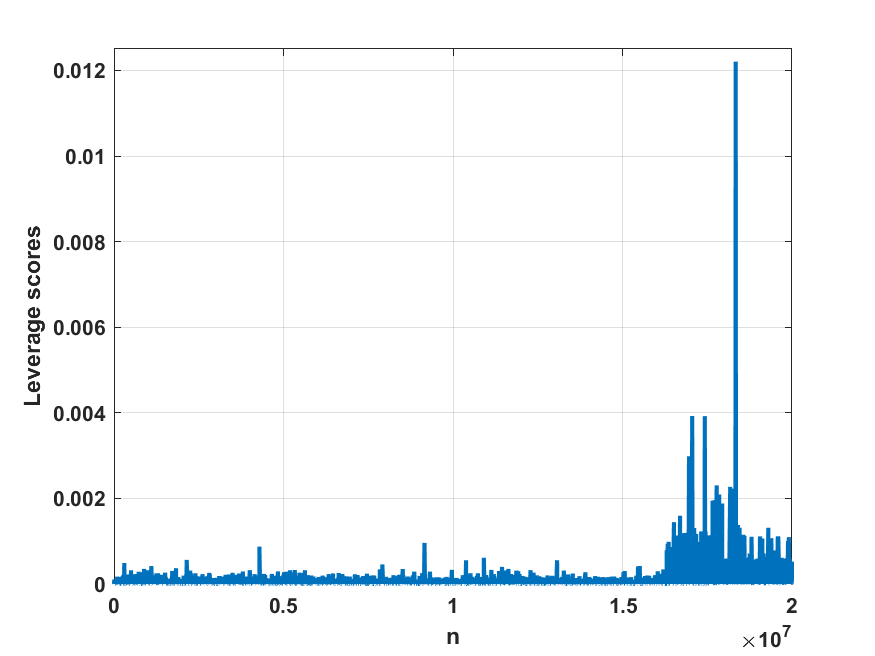}
            \caption[]{}    
            \label{fig:ls_real}
        \caption[]{Non-uniform distribution of leverage-scores of the real-data matrix with $m=20,000,000$ rows and $n=160$ columns extracted from weight data of a huge knowledge graph in \cite{sybrandt2017moliere}.} 
        \label{fig:ls}
\end{figure*}
For the large real-data matrix, the exact leverage-scores are computed only once, because the exact algorithm does not have any stochastic components. However, when estimating the leverage-scores using $\mathtt{SALSA}$, we perform calculations for every combination of $(s_1, s_2)$, where $s_1 \in \{20,000, 40,000, \ldots, 200,000\}$ and $s_2 \in \{1, 2, \ldots, 10\}$. We repeated this process 50 times to remove the impact of potential fluctuations. Figure \ref{fig:Real} provides details on execution times and $\mathtt{MAPE}$ for both the exact algorithm and $\mathtt{SALSA}$ for each $(s_1, s_2)$ pair. 
\begin{figure*}[!ht]
        \centering
        \begin{subfigure}[b]{0.475\textwidth}
            \centering
            \includegraphics[width=\textwidth]{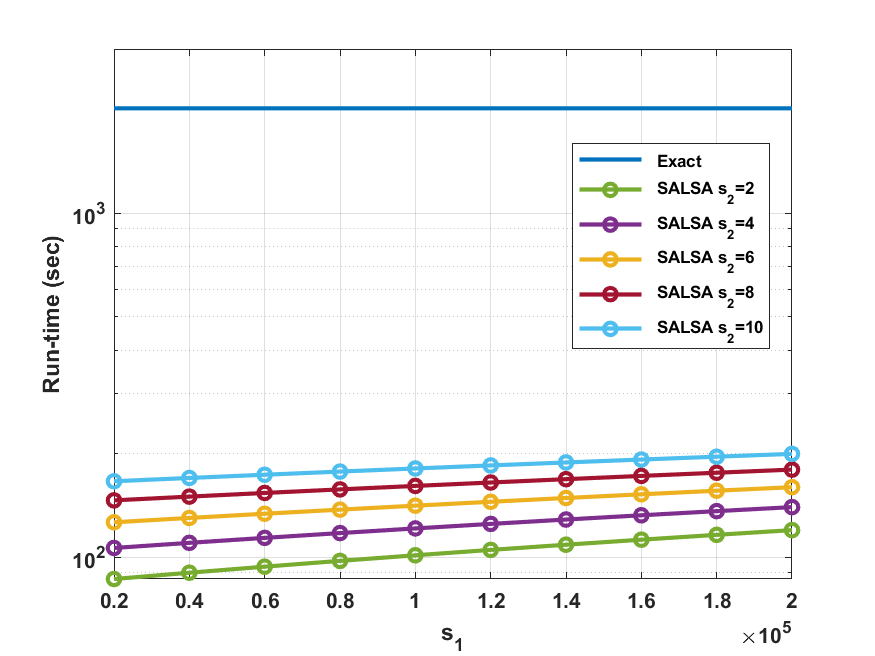}
            \caption[]{}    
            \label{fig:Real_s1_a}
        \end{subfigure}
        \hfill
        \begin{subfigure}[b]{0.475\textwidth}  
            \centering
            \includegraphics[width=\textwidth]{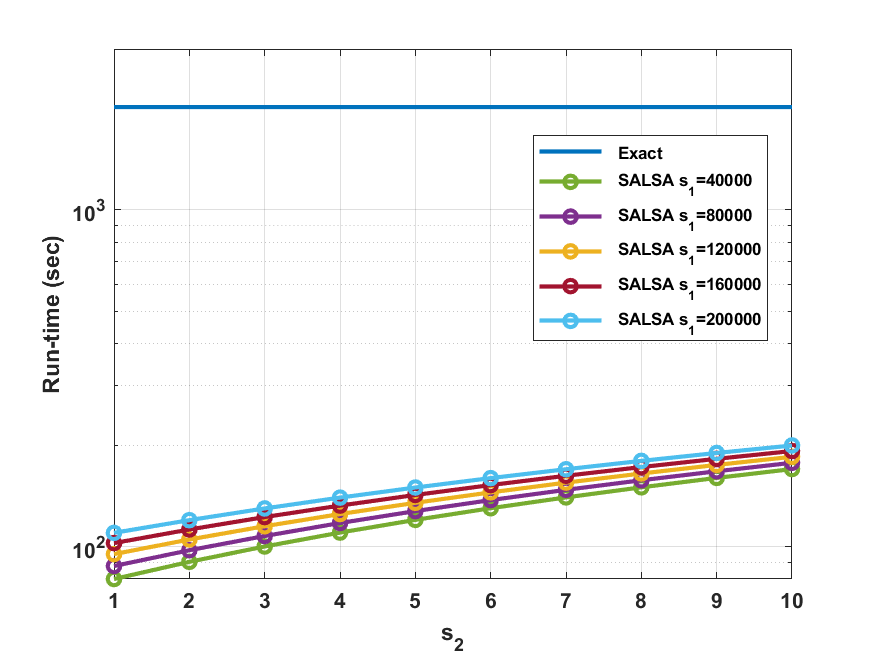}
            \caption[]{}    
            \label{fig:Real_s2_b}
        \end{subfigure}
        \vskip\baselineskip
        \begin{subfigure}[b]{0.475\textwidth}   
            \centering 
            \includegraphics[width=\textwidth]{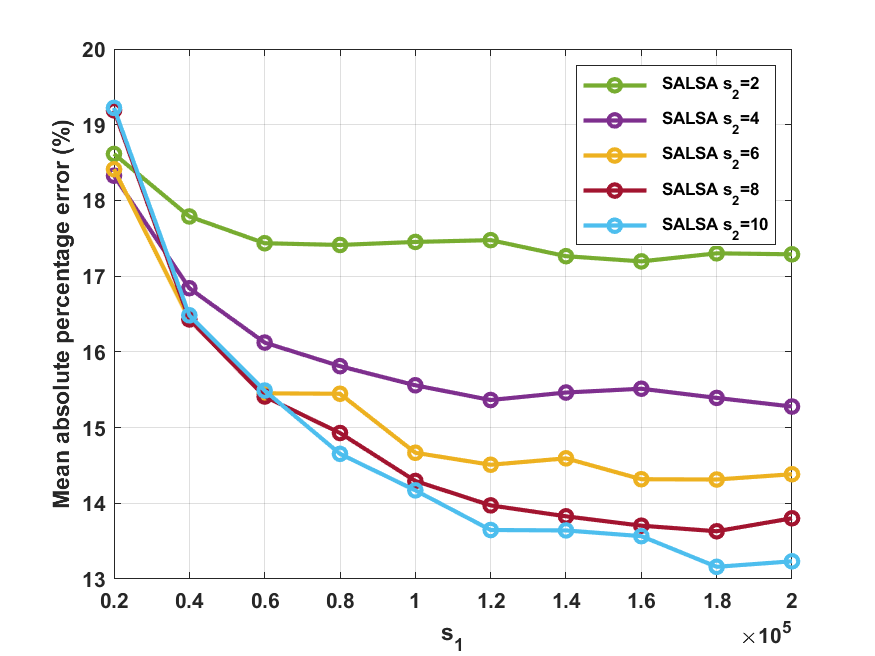}
            \caption[]{}    
            \label{fig:Real_s1_c}
        \end{subfigure}
        \hfill
        \begin{subfigure}[b]{0.475\textwidth}   
            \centering
            \includegraphics[width=\textwidth]{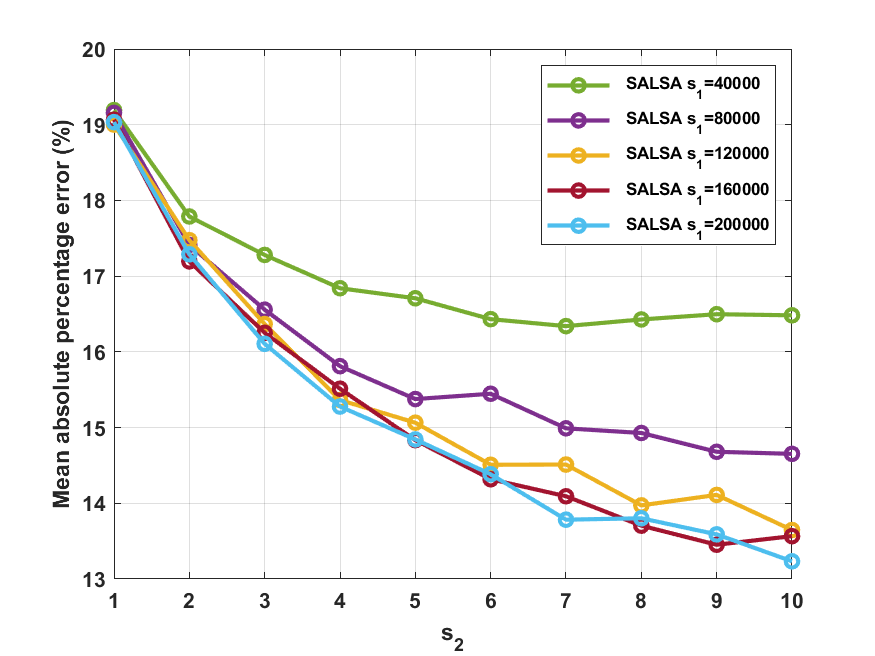}
            \caption[]{}    
            \label{fig:Real_s2_d}
        \end{subfigure}
        \caption[]{(\ref{fig:Real_s1_a}) shows the run times of approximating leverage-scores of the real-data matrix as a function of $s_2$ for different values of $s_1$, while (\ref{fig:Real_s2_b}) displays the run times as a function of $s_2$ for different fixed values of $s_1$. Run time for using the exact method in computing the leverage-scores is depicted in blue. The real-data matrix contains $m=20,000,000$ rows and $n=160$ columns, extracted from the edge-weight data matrix in \cite{sybrandt2017moliere}. Furthermore, (\ref{fig:Real_s1_c}) presents the $\mathtt{MAPE}$ against $s_1$, and (\ref{fig:Real_s2_d}) depicts the same error metric against $s_2$ between the exact and estimated leverage-scores.} 
        \label{fig:Real}
\end{figure*}

Figures \ref{fig:Real_s1_a} and \ref{fig:Real_s2_b} illustrate the efficiency of $\mathtt{SALSA}$ in estimating leverage-scores based on the values of $s_1$ and $s_2$ when compared with the exact algorithm. Evidently, as $s_1$ or $s_2$ increases, $\mathtt{SALSA}$ needs to handle a larger sketch of the original data matrix, subsequently requiring more time for leverage-score estimation. This consistent linear trend remains noticeable regardless of variations in either $s_1$ or $s_2$. Furthermore, these plots highlight that variations in $s_2$  have a greater impact than increases in $s_1$. The gaps between the curves representing different levels of $s_2$ in Figures \ref{fig:Real_s1_a} are significantly larger than the differences between the curves for varying levels of $s_1$  in Figure \ref{fig:Real_s2_b}. This observation aligns with the innovation introduced in $\mathtt{SALSA}$, which employs column-wise sketching with $s_2$ to enhance the numerical performance of estimating leverage-scores. Whereas exact leverage-score computation takes $2,017$ seconds (approximately $33$ minutes) for the whole data matrix, the most time-demanding pair, $s_1=200,000$ and $s_2=10$, requires only $3.3$ minutes. This finding indicates an approximately 10-fold increase in execution efficiency. The least time-demanding pair, $s_1= 20,000$ and $s_2=1$, only requires about one minute for leverage-score estimation.

Figures \ref{fig:Real_s1_c} and \ref{fig:Real_s2_d} explain the reduction in the percentage error when estimating leverage-scores as functions of $s_1$ and $s_2$. $\mathtt{MAPE}$, defined in (\ref{percentageError}), measures the$\mathtt{MAPE}$ between the estimated leverage-scores and those computed exactly. As both $s_1$ and $s_2$ increase, the approximate matrix incorporates more information from the original matrix, leading to more accurate leverage-score estimations. For instance, when $s_1$ and $s_2$ hit their highest values, namely, $s_2=10$ or $s_1=200,000$, the lowest error curves are observed. The highest error of $19.69\%$ occurred when $s_1=20,000$ and $s_2=10$, whereas the lowest $\mathtt{MAPE}$ of $13.16\%$ occurred for $s_1=180,000$ and $s_2=10$. With a worst-case $\mathtt{MAPE}$ of approximately $20\%$ and a 5-fold reduction in execution time, these results signify the efficiency of $\mathtt{SALSA}$ in estimating leverage-scores.

Another interesting observation about $\mathtt{SALSA}$ when applied to real-data matrices, as depicted in Figure \ref{fig:Real_s1_c}, is that increasing the value of $s_1$ while keeping $s_2$ fixed does not appear to lead to significant improvements in the mean-absolute-error levels. However, Figure \ref{fig:Real_s2_d} reveals that by increasing the value of $s_2$ while holding $s_1$ constant, achieving a reduction in the$\mathtt{MAPE}$.

% Subfigures \ref{fig:SALSA_time_error_a} and \ref{fig:SALSA_time_error_b} clearly demonstrate that increasing $s_1$ and $s_2$ which enlarges the size of row and column sketches and consequently, $\mathtt{SALSA}$ needs to solve a bigger OLS problem, the execution time converges to an upper bound. This scalability phenomenon is an important aspect of $\mathtt{SALSA}$ that the method is capable of dealing with even larger problems without requiring size-dependent execution time. However, based on the plots in Subfigures \ref{fig:SALSA_time_error_c} and \ref{fig:SALSA_time_error_d}, as the values of $s_1$ and $s_2$ increase, the amount of percentage relative error decrease, which implies that the estimated leverage scores are closer to the exact values.  This fact potentially creates a situation in which we can decide about a sweat spot or a trade-off in equilibrium values for $s_1$ and $s_2$ that we obtain a relatively small amount of percentage deviation from the exact leverage scores with moderately lower execution time. 

	%-----------------------
	% LSARMA
	%-----------------------
	
    \section{$\mathtt{LSARMA}$: Application of $\mathtt{SALSA}$ in Big Time Series Data}
\label{sec:lsarma}

This section demonstrates an application of the theory established in Section~\ref{chap_theo} regarding approximating leverage-scores. We employ $\mathtt{SALSA}$ in the development of a novel algorithm designed for fitting autoregressive moving average ($\mathtt{ARMA}$) models in scenarios involving big data. To provide context, we begin with a concise overview of the necessary background information on $\mathtt{ARMA}$ models and the $\mathtt{MLE}$ of their parameters.

\subsection{$\mathtt{ARMA}$ Models}
A time series $\{X_t; \; t=0,\pm 1, \pm 2, \ldots \}$ is a sequence of random variables indexed according to the order they are observed in time. The main goal of time series analysis is to create statistical models that predict the future behavior of a system. These models have established their efficacy and benefits in the modeling and analysis of stochastic dynamic systems. Their popularity is increasing for a variety of applications, ranging from supply chains and energy systems to epidemiology and engineering problems \cite{abolghasemi_2020,Eshragh2019,Eshragh2020,fahimnia_2015}.

Each observed value of $X_t$ is labeled $x_t$. A time series $X_t$ is considered (weakly) stationary if its mean function $\mathbb{E}[X_t]$ is constant, and the autocovariance function $Cov(X_t,X_{t+h})$ is independent of time.

$\mathtt{ARMA}$ models are a class of statistical models designed for the analysis and forecasting of stationary time series data. These models are characterized by their order, denoted as $\mathtt{ARMA}(p, q)$, which specifies the order of the autoregressive component ($p$) and the moving average component ($q$). In practice, a specific $\mathtt{ARMA}$ model is fully identified by estimating the parameters $\mathbf{\phi}=(\phi_1, \ldots, \phi_p)$ and $\mathbf{\theta}=(\theta_1, \ldots, \theta_q)$, as illustrated in
\begin{equation*}
    X_t = \phi_1 X_{t-1} + \ldots +\phi_p X_{t-p} + \theta_1 W_{t-1} + \ldots+ \theta_q W_{t-q} + W_t,
\end{equation*}
where $\phi_p \neq 0$ and $\theta_q \neq 0$. Additionally, $\{W_t; t=0,\pm1, \pm2, \ldots \}$ is a Gaussian white noise process with the properties $\mathbb{E}[W_t] = 0$ and $Cov(W_t,W_s) = \delta_{ts} \sigma^2_W$, where $\delta_{ts}$ denotes the Kronecker delta. Note we can assume $\mathbb{E}[X_t] = 0$ without any loss of generality. If $\mathbb{E}[X_t] = \mu \neq 0$, a simple transformation can be applied by replacing $X_t$ with $Y_t = X_t - \mu$ to hold this assumption. Furthermore, an $\mathtt{ARMA}(p,q)$ process is said to be invertible if the time series can be written as
\begin{equation}
    W_t = \sum_{j=0}^{\infty} \pi_j X_{t-j}, \quad \, \text{where  } \: \pi_0 = 1 \text{ and } \sum_{j=0}^{\infty} |\pi_j| < \infty.
    \label{eqn:invertible_process}
\end{equation}
Equivalently, the process is considered invertible if and only if the roots of the $\mathtt{MA}$ polynomial $\Theta_{q}(z) = 1+\theta_1 z + \ldots + \theta_q z^q$, for $z \in \mathbb{C}$, lie outside the unit circle. We assume all subsequent $\mathtt{ARMA}$ processes are causal, invertible, and stationary unless explicitly stated otherwise.

A total of $p + q + 3$ parameters need to be estimated when fitting an $\mathtt{ARMA}$ model, which includes the orders $p$ and $q$, coefficients $\phi_i$ and $\theta_i$, and the variance of the white noise $\sigma^2_W$. The estimation of the orders $p$ and $q$ requires the analysis of the Autocorrelation Function ($\mathtt{ACF}$) and Partial Autocorrelation Function ($\mathtt{PACF}$) plots. These functions help identify the number of past observations that significantly influence the current value in the time series (see \cite{shumwayTimeSeriesAnalysis2017} for details). The estimation of coefficients $\phi_i$, $\theta_i$, and $\sigma^2_W$ through the maximum likelihood technique is discussed in the next section.

\subsection{Maximum Likelihood Estimation of the Parameters} 
Consider a time series realization $x_1,\ldots,x_n$ generated by an $\mathtt{ARMA}(p,q)$ process, where the orders $p$ and $q$ are known and the length of the series $n$ is significantly larger than $p$ and $q$. This section considers the estimation of the parameters $\phi_1,\ldots,\phi_p,\theta_1,\ldots,\theta_q,$ and $\sigma^2_W$ through the $\mathtt{MLE}$ method. Obtaining an analytical expression for the estimates is interactable due to the complicated and nonlinear nature of the log-likelihood function, which is defined as
$$
\log ( f_{X_1,\ldots,X_n}(x_1,\ldots,x_n;\phi_1,\ldots,\phi_p,\theta_1,\ldots,\theta_q,\sigma^2_W)),
$$
where  $f$ is the joint probability distribution function of the random variables $X_1,\ldots,X_n$. 
As a result, numerical optimization techniques are typically used to approximate these estimates. 
These techniques explore parameter values that maximize the log-likelihood function so that the model fits the observed data as accurately as possible. However, we can demonstrate that the conditional log-likelihood function, described below, shares similarities with the log-likelihood function of a linear regression model:
\begin{align*}
    &\log(f_{X_{p+1},\ldots,X_n | X_1,\ldots,X_p,W_1,\ldots,W_q} (x_{p+1},\ldots,x_n;  \phi_1,\ldots,\phi_p,\theta_1,\ldots,\theta_q,\sigma^2_W | x_1,\ldots,x_p,w_1,\ldots,w_q)) \\
    & \quad \quad = - \frac{n-p}{2} \log (2\pi) - \frac{n-p}{2} \log(\sigma^2_W) \\ 
    & \quad \quad \quad \quad - \sum^n_{t=p+1} \frac{(x_t-\phi_1 x_{t-1} - \ldots - \phi_p x_{t-p} - \theta_1 w_{t-1} - \ldots - \theta_q w_{t-q})^2 }{2\sigma^2_W} ,
\end{align*}
where the log-likelihood has been conditioned on the first $p$ observations, equivalent to the order of the $\mathtt{AR}$ component, and the first $q$ white noises, equivalent to the order of the $\mathtt{MA}$ component \cite{hamiltonTSA}. 
This similarity allows for the use of regression-based optimization techniques in estimating the parameters of an $\mathtt{ARMA}$ model. Hence, the conditional $\mathtt{MLE}$ ($\mathtt{CMLE}$) of the coefficients and the variance can be obtained through an $\mathtt{OLS}$ regression of $x_t$ against its own lagged values up to order $p$ and the previous white noise terms up to order $q$. However, in practice, the previous white noise terms are not directly available and must be estimated from the available data before this approach can be applied.

J. Durbin's method \cite{durbinEfficientEstimationParameters1959}, which was originally developed for fitting $\mathtt{MA}$ models, has been extended to handle $\mathtt{ARMA}$ processes. The approach involves fitting an $\mathtt{AR}$ model to the data and then using the residuals of this model to estimate the parameters of the $\mathtt{MA}$ model. We refer to this method as Durbin's algorithm here. 
Durbin's algorithm allows for the estimation of the unobserved white noise terms, hence enabling the $\mathtt{OLS}$ approach to be used when the noise terms are unavailable. Equation~\cref{eqn:invertible_process} shows any invertible $\mathtt{ARMA}(p,q)$ process can be equivalently represented as an $\mathtt{AR}(\infty)$ process. However, in practice, due to finite sample sizes, Durbin's algorithm first fits an $\mathtt{AR}$ model with a sufficiently large order, denoted as $\tilde{p}$ (where $\tilde{p} \ll n$), to the data. Next, the algorithm approximates the values of the white noise term $W_t$ by computing the residuals of the fitted $\mathtt{AR}$ model,
\begin{equation}
    \hat{w}_t = x_t -  \sum_{j=1}^{\tilde{p}} \hat{\pi}_j x_{t-j}.
    \label{eqn:long_AR}
\end{equation}
Subsequently, as the next step, the algorithm approximates the $\mathtt{CMLE}$ of the coefficient vector $\bm{\psi}_{n,p,q} = (\phi_1,\ldots,\phi_p,\theta_1,\ldots,\theta_q)^\transpose$ as
\begin{equation}
    \hat{\bm{\psi}}_{n,p,q} = (\XX^\transpose_{n,p,q} \XX_{n,p,q})^{-1} \XX^\transpose_{n,p,q} \xx_{n,p,q},
    \label{eqn:coef_vector}
\end{equation}
where the data matrix $\XX_{n,p,q}$ is defined as
\begin{align}
    \left(\begin{array}{cccc cccc}
	\medskip x_{q+\tilde{p}} & \cdots & x_{q+\tilde{p}-p+1} 			& \hat{w}_{q+\tilde{p}} & \cdots	& \hat{w}_{\tilde{p}+1}		\\
	\medskip x_{q+\tilde{p}+1} & \cdots & x_{q+\tilde{p}-p+2} 			& \hat{w}_{q+\tilde{p}+1}	& \cdots	& \hat{w}_{\tilde{p}+2}	\\ 
	\medskip \vdots & \ddots & \vdots 		& \vdots  	& \ddots 	& \vdots	\\	
	\medskip x_{n-1} & \cdots & x_{n-p} & \hat{w}_{n-1}	& \cdots	& \hat{w}_{n-q}
	\end{array} \right),
	\label{eqn:data_matrix}
\end{align}
and $\xx_{n,p,q} := (x_{q+\tilde{p}+1},x_{q+\tilde{p}+2},\ldots,x_n)^\transpose$. Moreover, the $\mathtt{CMLE}$ of $\sigma^2_W$ is obtained  by 
$$\hat{\sigma}^2_W =\frac{ \| \rr_{n,p,q} \|^2}{(n-p-q)},$$ 
where $\hat{\rr}_{n,p,q} := \XX_{n,p,q} \hat{\bm{\psi}}_{n,p,q} - \xx_{n,p,q}$ is the residual vector. Hannan and Rissanen \cite{hannanRecursiveEstimationMixed1982} extended this algorithm with an additional set of trimming steps to improve the initial parameter~estimates.

\subsection{Choosing a Sufficiently Large $\tilde{p}$} 

The equivalence between invertible $\mathtt{ARMA}(q)$ and $\mathtt{AR}(\infty)$ models might imply larger values of $\tilde{p}$ would consistently outperform smaller ones. However, this equivalence does not consider finite sample sizes. As noted in \cite{broersenAutoregressiveModelOrders2000}, \\
\begin{quote}  %\enquote{
\textit{Practice and simulations have shown that the best $\mathtt{AR}$ order in estimation is finite and depends on the true process parameters and the number of observations.}
\end{quote} %} \\
\noindent Since the first procedure in estimating $\tilde{p}$ in  \cite{durbinEfficientEstimationParameters1959}, various methods have been developed to select the most suitable value.

As an example, Hannan and Rissanen \cite{hannanRecursiveEstimationMixed1982} proposed the use of the Bayesian information criterion ($\mathtt{BIC}$) to determine the optimal value of $\tilde{p}$. This criterion can be applied to a regression model with $k$ coefficients. In this context, the $\mathtt{MLE}$ for the variance is denoted as $\hat{\sigma}_k^2 = \mathtt{SSE}(k)/n$, where $\mathtt{SSE}(k)$ represents the residual sum of squares under the model with $k$ regression coefficients. Then,
\begin{align}
	  \medskip \mathtt{BIC} := \log(\hat{\sigma_k^2}) + \frac{k\log(n)}{n}, 
	  \label{eqn:BIC}
\end{align}
with the value of $k$ yielding the minimum $\mathtt{BIC}$ specifying the best regression model. Additionally, Broersen \cite{broersenAutoregressiveModelOrders2000} recommended selecting the optimal order as the one that minimizes the $\mathtt{GIC}(p,1)$ criterion, where $p$ ranges from zero to infinity. Here, $\mathtt{GIC}$ stands for a generalization of Akaike's information criterion, offering a practical approach to determining the most suitable order for $\mathtt{AR}$ models. More recently, the rolling average algorithm, introduced by Eshragh et al. \cite{eshraghRollage2020}, is a method for selecting an optimal $\tilde{p}$ when fitting $\mathtt{ARMA}$ models, especially in the context of big time series datasets.

\subsection{The $\mathtt{LSARMA}$ Algorithm}
We now introduce $\mathtt{LSARMA}$, as presented in Algorithm~\ref{alg:LSARMA}. This algorithm is based on the theoretical foundations outlined in Section~\ref{chap_theo} and extends the $\mathtt{LSAR}$ algorithm \cite{eshraghLSAREfficientLeverage2019}. 
$\mathtt{LSAR}$ uses leverage-score sampling and the Toeplitz structure of the data matrix to approximate $\mathtt{AR}$ models for large time series datasets, and $\mathtt{LSARMA}$ offers a generalization to $\mathtt{ARMA}$ models.

%-----------------------------------------------------------
\begin{algorithm}[tb] 
	\caption{$\mathtt{LSARMA}$}	
	\label{alg:LSARMA}
	\begin{algorithmic}
		\STATE {\bfseries Input:} Time series data $\{x_1,\ldots,x_n\}$, relatively large order $\bar{p} \ll n$, an estimate of the $\mathtt{MA}$ order $q$;
        \vspace{0.2cm}
        \STATE Apply the $\mathtt{LSAR}$ algorithm using the $BIC$ criterion, as in \cref{eqn:BIC}, to choose a suitable $\tilde{p}$ to obtain the white noise estimates $\hat{w}_t$ as in \cref{eqn:long_AR};
        \STATE Construct the data matrix $\XX_{n,0,q}$ as in \cref{eqn:data_matrix} and estimate its leverage-scores using \texttt{SALSA};
        \FOR{$h=1$ {\bfseries to} $\bar{p}$}
            \STATE Use $\mathtt{SALSA}$ to estimate the leverage-scores of $\XX_{n,h,q}$;
            \STATE Use the estimated leverage-scores to find the counterpart sketched solution as given in \cref{eqn:phi_hat_defined}, of $\hat{\bm{\psi}}_{n,h,q}$ as in \cref{eqn:coef_vector}
            \STATE Compute the $\mathtt{PACF}_h$
        \ENDFOR
        \STATE Plot $\mathtt{PACF}_h$ against $h$ to estimate the $\mathtt{AR}$ order $p$. 
		\STATE {\bfseries Output:} Estimated orders $p$, $q$ and parameter vector $\hat{\bm{\psi}}_{n,p,q}$.  
	\end{algorithmic}
\end{algorithm}

\begin{remark}[$\mathtt{LSARMA}$ Computational Complexity, \cite{eshraghLSAREfficientLeverage2019}, Theorem 6]
    The computational bottleneck of the $\mathtt{LSARMA}$ algorithm is due to the use of the $\mathtt{LSAR}$ algorithm. Consequently, the worst-case time complexity of the $\mathtt{LSARMA}$ algorithm is the same as that of the $\mathtt{LSAR}$ algorithm, that is,
    \begin{equation*}
        \bigO{n\tilde{p} + \frac{\tilde{p}^4 \log{\tilde{p}}}{ \varepsilon_1^2 }},
    \end{equation*}
    where $\varepsilon_1$ is introduced in \cref{thm:michael}. Note exact $\mathtt{ARMA}(p,q)$ fitting requires $\bigO{n\tilde{p}^3}$. 
\end{remark}

\subsection{Numerical Results: Big Time Series Data} \label{sec:lsarma_numerical}

The numerical results comparing the execution time and percentage error of parameter estimation using $\mathtt{LSARMA}$ versus an exact method for synthetic $\mathtt{MA}(q)$ data with $n=5,000,000$ entries are illustrated in Figure \ref{fig:SALSA_MA2}. Although the exact method is executed only once, because it lacks any randomization in its procedure, the $\mathtt{LSARMA}$ method is run for $25$ iterations due to the inclusion of a stochastic sketching module within its procedure. Consequently, the results of $\mathtt{LSARMA}$ are averaged to facilitate a meaningful comparison with the exact method.

The benefits of employing $\mathtt{LSARMA}$ for estimating the model parameters are evident in Figure \ref{fig:SALSA_MA2_a}. As the value of $q$ increases, both methods show positive slopes in their computational time curves, but $\mathtt{LSARMA}$ consistently requires significantly fewer resources. To grasp the computational advantage, note the average execution time for the exact method remains at $106,875$ seconds (about $30$ hours) for all values of $q$, whereas $\mathtt{LSARMA}$ averages at only $3,526$ seconds (about one hour), making it approximately $30$ times faster. Furthermore, when considering all values of $q$, the maximum time required for the exact method to find the parameters is a substantial $89$ hours, whereas $\mathtt{LSARMA}$ accomplishes the task in a mere $3$ hours.

These results presented in Figure \ref{fig:SALSA_MA2_a} underscore the superior performance and efficiency of $\mathtt{LSARMA}$ over the exact method, particularly as the complexity of the data increases. This advantage becomes even more significant as the value of $q$ grows, highlighting the practicality, scalability, and computational gain of using $\mathtt{LSARMA}$ for parameter estimation in $\mathtt{MA}(q)$ data.

\begin{figure}[!ht]
\vskip 0.2in
     \centering
     \begin{subfigure}[b]{0.45\textwidth}
         \centering
         \includegraphics[width=1\columnwidth]{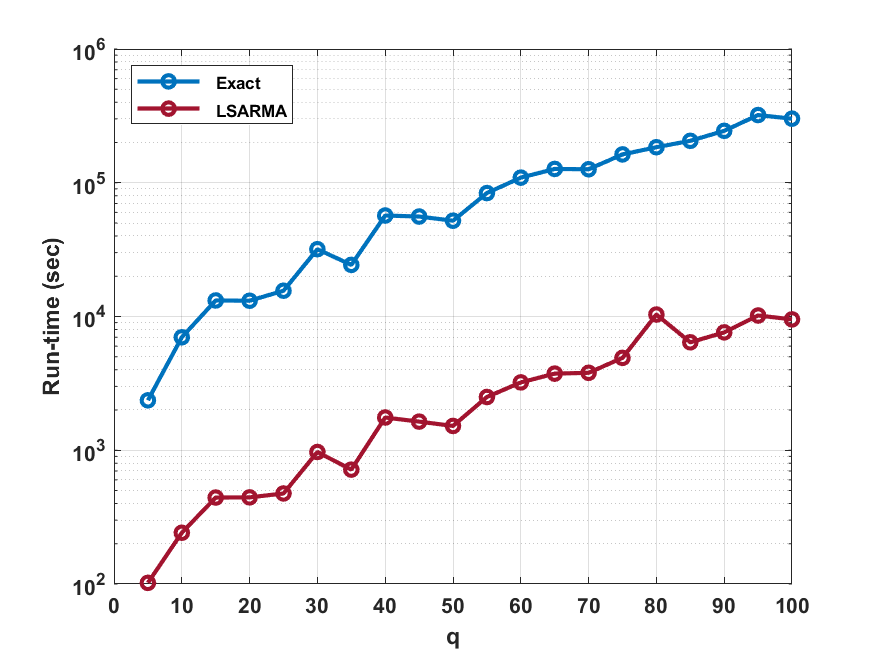}
         \caption{}
         \label{fig:SALSA_MA2_a}
     \end{subfigure}
     %\hfill
     \begin{subfigure}[b]{0.45\textwidth}
         \centering
         \includegraphics[width=1\columnwidth]{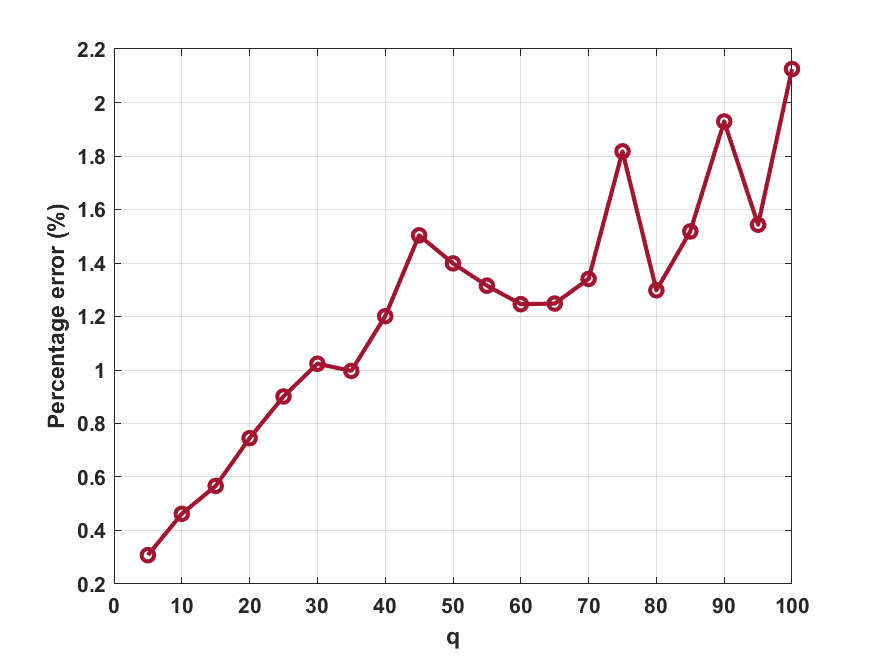}
         \caption{}
         \label{fig:SALSA_MA2_b}
     \end{subfigure}
     
        \caption{(\ref{fig:SALSA_MA2_a}) Run times in computing exact and estimated parameters and (\ref{fig:SALSA_MA2_b}) percentage error in the difference between exact and estimated parameters for synthetic $\mathtt{MA}(q)$ data with $n=5,000,000$ entries.}
        \label{fig:SALSA_MA2}
    \vskip -0.2in
\end{figure}

To further demonstrate the efficiency of $\mathtt{LSARMA}$, we visualize its computational accuracy in comparison to the exact method by analyzing the percentage error between the parameters obtained from both methods in Figure \ref{fig:SALSA_MA2_b}. To calculate the percentage error for each value of $q$, we use the Euclidean norm as follows:
\begin{equation}
  \text{Percentage error} =   \frac{\| \bm{\psi}_{n,p,q} - \hat{\bm{\psi}}_{n,p,q} \|}{\|\bm{\psi}_{n,p,q}\|} \times 100\%. 
  \label{percentageErrorMA}
\end{equation}
Although a slight upward trend occurs in the percentage error as $q$ increases, it remains remarkably low, with the maximum percentage error across all values of $q$ being only $2.1\%$. This outstanding result showcases the accuracy of parameter estimation in large time series data using $\mathtt{LSARMA}$, even when working with just a fraction of the data. Occasional fluctuations are observed for $q$ values greater than $75$. We attribute these fluctuations primarily to the resource allocation of MATLAB. However, these fluctuations are minimal - less than $0.2\%$ - and can be considered negligible given the data size.

Considering both its computational efficiency and accuracy over the exact method in Figure \ref{fig:SALSA_MA2}, $\mathtt{LSARMA}$ represents itself as a reliable and robust method for estimating parameters in $\mathtt{MA}$ models. With a significant reduction in execution time, approximately 30 times faster than the exact method, and only a marginal deviation of $2.1\%$ from the parameters obtained from the exact method, $\mathtt{LSARMA}$ proves to be an excellent choice for parameter estimation in large time series data.

Furthermore, the $\mathtt{LSARMA}$ algorithm finds application in estimating parameters for synthetically generated $\mathtt{ARMA}(p,q)$ models, with $p$ and $q$ selected from the set $\{5,10,20,40, 80\}$, and data models containing $n=5,000,000$ entries. These datasets are constructed with randomly generated data; hence, they contain varying levels of complexity in parameter estimation, due to the differing characteristics of the underlying models. For each specific model, the true parameters are determined exactly only once, because no stochasticity is involved in this process. Subsequently, the $\mathtt{LSARMA}$ algorithm is employed to estimate these parameters, and this estimation process is repeated ten times to account for variability. Both the execution time and the percentage of errors, as defined in (\ref{percentageErrorMA}), are computed for each method. The results of these experiments are visually presented in Figure \ref{fig:LSARMA_time_error}, providing insights into the algorithm's performance across a range of $\mathtt{ARMA}(p,q)$ models with distinct complexities.
\begin{figure*}[!ht]
        \centering
        \begin{subfigure}[b]{0.475\textwidth}
            \centering
            \includegraphics[width=\textwidth]{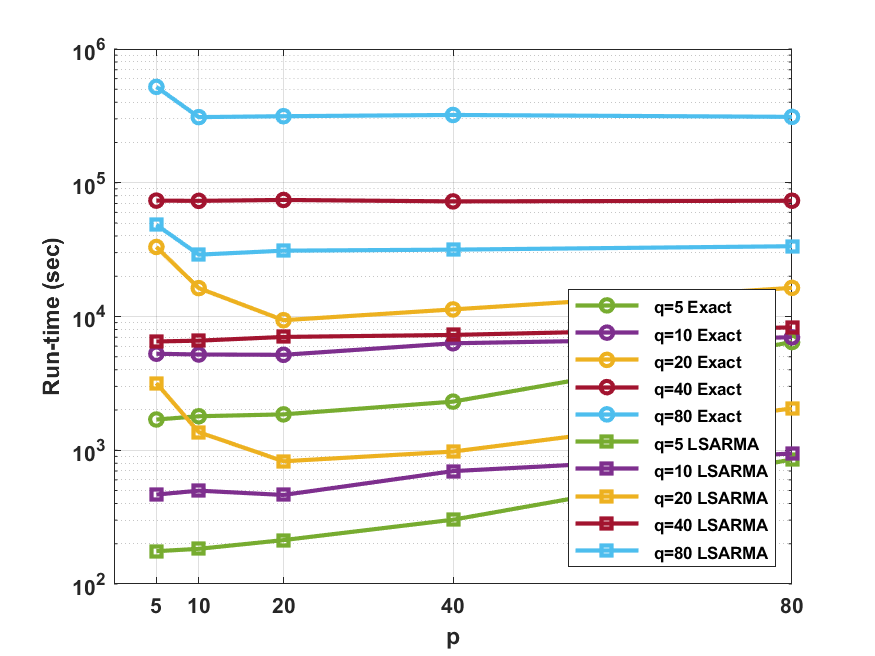}
            
            \caption[]{}    
            \label{fig:LSARMA_time_error_a}
        \end{subfigure}
        \hfill
        \begin{subfigure}[b]{0.475\textwidth}  
            \centering 
            \includegraphics[width=\textwidth]{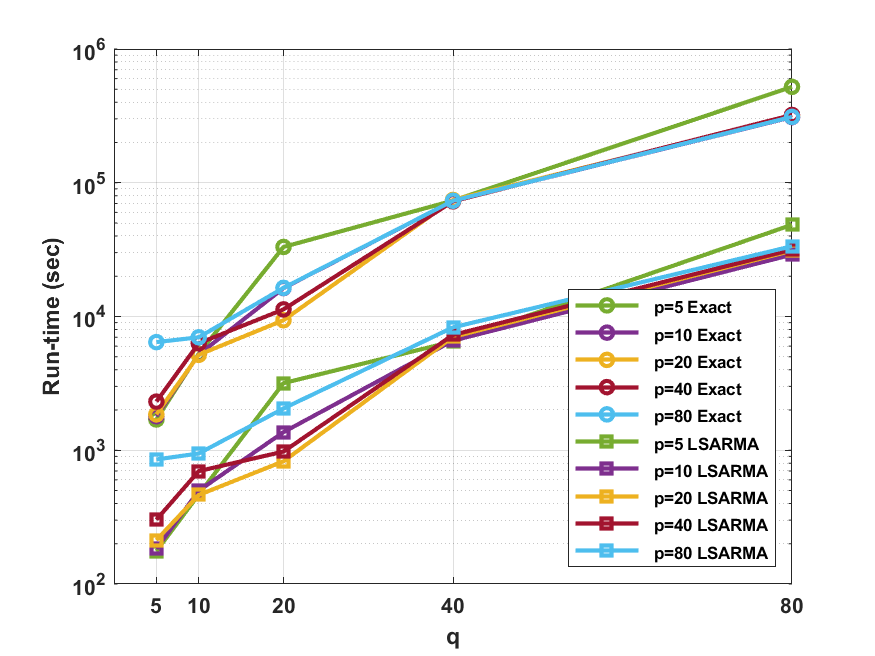}
            \caption[]{}    
            \label{fig:LSARMA_time_error_b}
        \end{subfigure}
        \vskip\baselineskip
        \begin{subfigure}[b]{0.475\textwidth}   
            \centering 
            \includegraphics[width=\textwidth]{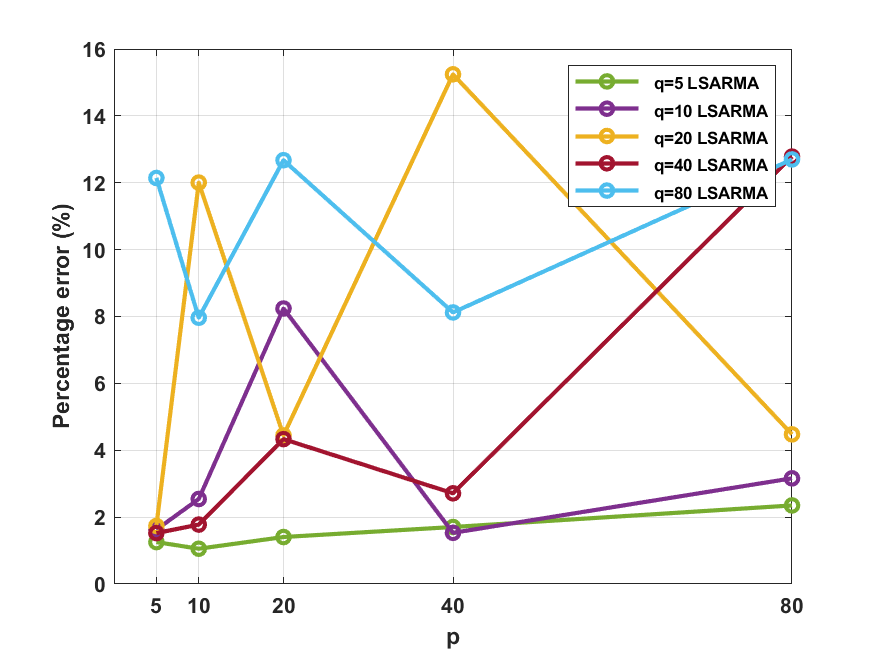}
            \caption[]{}    
            \label{fig:LSARMA_time_error_c}
        \end{subfigure}
        \hfill
        \begin{subfigure}[b]{0.475\textwidth}   
            \centering 
            \includegraphics[width=\textwidth]{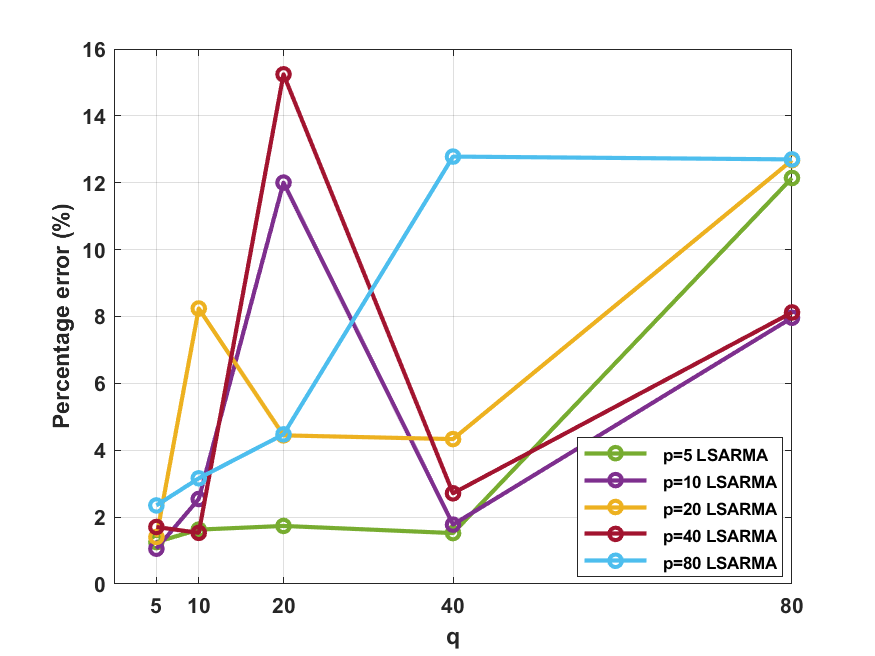}
            \caption[]{}    
            \label{fig:LSARMA_time_error_d}
        \end{subfigure}
        \caption[]{(\ref{fig:LSARMA_time_error_a}) shows the run times as functions of $p$ for different values of $q$ for computing both exact and estimated parameters, while (\ref{fig:LSARMA_time_error_b}) displays the run times as functions of $q$ for different fixed values of $p$. These results are obtained for synthetic $\mathtt{ARMA}(p,q)$ models  with $p,q \in \{5,10,20,40, 80\}$ and $n=5,000,000$ entries. Furthermore, (\ref{fig:LSARMA_time_error_c}) presents the$\mathtt{MAPE}$ against $p$, and (\ref{fig:LSARMA_time_error_d}) depicts the same error metric against $q$ between the exact and  $\mathtt{LSARMA}$ algorithms.} 
        \label{fig:LSARMA_time_error}
\end{figure*}

From a time-efficiency perspective, Figures \ref{fig:LSARMA_time_error_a} and \ref{fig:LSARMA_time_error_b} provide a comprehensive illustration of the computational time required to obtain exact and estimated parameters using the exact and $\mathtt{LSARMA}$ algorithms, respectively. These figures present insights into how this computational time changes across different combinations of $p$ and $q$ values. In both scenarios, note the most time-intensive computations are associated with the $\mathtt{ARMA}(5,80)$ model. Specifically, the exact algorithm demands $521,132$ seconds (approximately $145$ hours) to complete, whereas the $\mathtt{LSARMA}$ algorithm accomplishes the same task in significantly less time, taking only $48,495$ seconds (roughly $13$ hours). This remarkable efficiency is evidenced by the fact that, for the most computationally demanding model, $\mathtt{LSARMA}$ is approximately $11$ times faster than the exact algorithm. Conversely, the least computationally demanding model, $\mathtt{ARMA}(5,5)$, can be estimated in $47$ minutes using the exact algorithm and only $3$ minutes with the $\mathtt{LSARMA}$ algorithm. Here, $\mathtt{LSARMA}$ shows approximately a $16$-times increase in efficiency.

Additionally, Figures \ref{fig:LSARMA_time_error_a} and \ref{fig:LSARMA_time_error_b} provide insights into execution times by fixing values of either $q$ or $p$ while varying the other variable. Particularly, Figure \ref{fig:LSARMA_time_error_a} demonstrates nearly constant execution times when $q$ is fixed and $p$ values are varied. Conversely, Figure \ref{fig:LSARMA_time_error_b} reveals an increasing trend in execution times when $p$ is held constant, and $q$ values are changed. These observations shed light on the complicated interplay between $p$ and $q$ in determining the computational time required for parameter estimation. Note the execution time of the exact algorithm and $\mathtt{LSARMA}$ share the same color, but different markers of $\bigcirc$ and $\Box$, respectively.

Figures \ref{fig:LSARMA_time_error_c} and \ref{fig:LSARMA_time_error_d} provide graphical representations of the percentage error between parameters estimated by the exact and $\mathtt{LSARMA}$ algorithms, as defined in (\ref{percentageErrorMA}). Given the random construction of models for each combination of $(p,q)$, these models inherently possess varying degrees of complexity, resulting in fluctuating error curves as $p$ or $q$ change. Consequently, the observed fluctuations in the error curves can be attributed to these inherent variations. Particularly, the maximum percentage error, amounting to $15.24\%$, is observed in parameter estimation for the $\mathtt{ARMA}(40,20)$ model. Conversely, the minimum percentage error, at only $1.05\%$, is observed for the $\mathtt{ARMA}(10,5)$ model. On average, across all models considered, the percentage error in parameter estimation between the exact and $\mathtt{LSARMA}$ algorithms is approximately $4.63\%$, with a standard deviation of $5.58\%$. 

These statistics offer valuable insights into the overall accuracy and consistency of parameter estimation achieved by the $\mathtt{LSARMA}$ algorithm, when compared with the exact method. With an average percentage error of just $4.63\%$, coupled with the significant efficiency gains of being approximately $11$ times faster in terms of execution time, the $\mathtt{LSARMA}$ algorithm clearly demonstrates its superior capability in reliably estimating the parameters of $\mathtt{ARMA}(p,q)$ models.

	%-----------------------
	% Conclusion
	%-----------------------
	
	%\input{Conclusion.tex}
	
	%-----------------------
	% Acknowledgement
	%-----------------------
	
	%\input{Acknowledgement.tex}
	
	%-----------------------

	%-----------------------------------------------

	%-----------------------
	% Appendix
	%-----------------------

%\begin{appendix}
\newpage
\onecolumn
\appendix
\allowdisplaybreaks

%---------------------------------------------
\section{Technical Lemmas and Proofs}
\label{Sec:Proofs}
%---------------------------------------------

%---------------------------------------------
\subsection{Proof of \cref{thm:General_Lev_Score_Recursion} (Exact Recursive Leverage Scores)}
%---------------------------------------------
We first present \cref{lem:Matrix_Inversion}, which we use in the proof of \cref{thm:General_Lev_Score_Recursion}.   

\begin{lemma}[Block Matrix Inversion Lemma \cite{golubMatrixComputations2013}]
    Consider the $2 \times 2$ block matrix 
    \begin{equation*}
        \MMM = \begin{bmatrix}
        \AA_{m,d} & \bb \\
        \bb^{\transpose} & c
        \end{bmatrix},
    \end{equation*}
    where $\AA_{m,d} \in \mathbb{R}^{m \times d}$, $\bb \in \mathbb{R}^{m \times 1}$ and $c \in \mathbb{R}$. If $\AA_{m,d}$ is invertible, the inverse of $M$ exists and can be calculated as follows:
    \begin{equation*}
                \MMM^{-1} = \frac1k \begin{bmatrix}
        k\AA_{m,d}^{-1} + \AA_{m,d}^{-1}\bb\bb^{\transpose}\AA_{m,d}^{-1} & -\AA_{m,d}^{-1} \bb \\
         -\bb^{\transpose} \AA_{m,d}^{-1} & 1
        \end{bmatrix},
    \end{equation*}
    where $k = c - \bb^{\transpose} \AA_{m,d}^{-1} \bb$. 
    \label{lem:Matrix_Inversion}
\end{lemma}

\subsubsection*{Proof of \cref{thm:General_Lev_Score_Recursion}}
\begin{proof}
    The $i^{th}$ leverage score of $\AA_{m,d}$ is given by the $i^{th}$ diagonal element of the hat matrix $\HH_{m,d}$, that is, 
    \begin{equation*}
    	\ell_{m,d} (i) = \HH_{m,d}(i,i) = \left[ \AA_{m,d} \left( \AA_{m,d}^{\transpose} \AA_{m,d} \right)^{-1} \AA_{m,d}^{\transpose} \right](i,i).
    \end{equation*}
    Now consider $\left( \AA_{m,d}^{\transpose} \AA_{m,d} \right)^{-1}$, where 
    \begin{equation*}
        \AA_{m,d}^{\transpose} \AA_{m,d} = \begin{bmatrix}
          \AA_{m,d-1}^{\transpose} \\
          \aa_{d-1}^\transpose
         \end{bmatrix}
         \left[ \AA_{m,d-1} \, \, \aa_{d-1} \right]
         = \begin{bmatrix}
          \AA_{m,d-1}^{\transpose} \AA_{m,d-1} & \AA_{m,d-1}^{\transpose} \aa_{d-1}\\
          \aa_{d-1}^{\transpose} \AA_{m,d-1} & \aa_{d-1}^{\transpose} \aa_{d-1}
         \end{bmatrix}.
    \end{equation*}
    Using the notation from \cref{lem:Matrix_Inversion}, we have 
    \begin{align*}
        k &= \aa_{d-1}^{\transpose} \aa_{d-1} - \aa_{d-1}^{\transpose} \AA_{m,d-1} \left( \AA_{m,d-1}^{\transpose} \AA_{m,d-1} \right)^{-1} \AA_{m,d-1}^{\transpose} \aa_{d-1} \\
        &= \aa_{d-1}^{\transpose} \left( \aa_{d-1} - \AA_{m,d-1} \bm{\phi}_{m,d-1} \right) \\
        &= -\aa_{d-1}^{\transpose} \rr_{m,d-1}.
    \end{align*}
    Hence,
    \begin{equation*}
        \left( \AA_{m,d}^{\transpose} \AA_{m,d} \right)^{-1} = \frac{-1}{\aa_{d-1}^{\transpose} \rr_{m,d-1}}
        \begin{bmatrix}
            -\aa_{d-1}^{\transpose} \rr_{m,d-1} \left( \AA_{m,d-1}^{\transpose} \AA_{m,d-1}\right)^{-1} + \bm{\phi}_{m,d-1} \bm{\phi}_{m,d-1}^{\transpose} & -\bm{\phi}_{m,d-1} \\
            -\bm{\phi}_{m,d-1}^{\transpose} & 1
        \end{bmatrix}.
    \end{equation*}
    Restricting attention to the $i^{th}$ diagonal element, we have
    \begin{align}
        \nonumber \ell_{m,d} (i) &= \AA_{m,d}(i,:) \left( \AA_{m,d}^{\transpose} \AA_{m,d} \right)^{-1} \AA^\transpose_{m,d}(:,i) \\
        %----------
        \nonumber &= \frac{-1}{\aa_{d-1}^{\transpose} \rr_{m,d-1}} 
        \begin{pmatrix} \AA_{m,d-1}(i,:) & \aa_{d-1}(i) \end{pmatrix} \\
        \nonumber &\quad \quad \times \begin{bmatrix}
            -\aa_{d-1}^{\transpose} \rr_{m,d-1} \left( \AA_{m,d-1}^{\transpose} \AA_{m,d-1}\right)^{-1} + \bm{\phi}_{m,d-1} \bm{\phi}_{m,d-1}^{\transpose} & -\bm{\phi}_{m,d-1} \\
            -\bm{\phi}_{m,d-1}^{\transpose} & 1
        \end{bmatrix}
        \begin{pmatrix}
            \AA^\transpose_{m,d-1}(:,i) \\
            \aa_{d-1}(i)
        \end{pmatrix} \\
        %---------
        \nonumber &= \frac{-1}{\aa_{d-1}^{\transpose} \rr_{m,d-1}} \begin{pmatrix}
                -\aa_{d-1}^{\transpose} \rr_{m,d-1} \AA_{m,d-1}(i,:) \left( \AA_{m,d-1}^{\transpose} \AA_{m,d-1} \right)^{-1} \\
                + \AA_{m,d-1}(i,:) \bm{\phi}_{m,d-1} \bm{\phi}_{m,d-1}^{\transpose} - \aa_{d-1}(i) \bm{\phi}_{m,d-1}^\transpose \\
                \\
             -\AA_{m,d-1}(i,:) \bm{\phi}_{m,d-1} + \aa_{d-1}(i)
        \end{pmatrix}^\transpose \\
        \nonumber &\quad \quad \times
        \begin{pmatrix}
            \AA^\transpose_{m,d-1}(:,i) \\
            \aa_{d-1}(i)
        \end{pmatrix} \\
        %---------
        \nonumber &= \frac{-1}{\aa_{d-1}^{\transpose} \rr_{m,d-1}} \bigg(
            -\aa_{d-1}^{\transpose} \rr_{m,d-1} \AA_{m,d-1}(i,:) \left(\AA_{m,d-1}^{\transpose} \AA_{m,d-1} \right)^{-1} \AA^\transpose_{m,d-1}(:,i) \\
            \nonumber &\quad \quad + \AA_{m,d-1}(i,:) \bm{\phi}_{m,d-1} \bm{\phi}_{m,d-1}^{\transpose} \AA^\transpose_{m,d-1}(:,i) - \aa_{d-1}(i) \bm{\phi}_{m,d-1}^{\transpose} \AA^\transpose_{m,d-1}(:,i) \\
            \label{eqn:ls_rec} & \quad \quad \quad \quad - \AA_{m,d-1}(i,:) \bm{\phi}_{m,d-1} \aa_{d-1}(i) + \aa_{d-1}(i)^2
        \bigg). 
    \end{align}
    Note we have
    \begin{align}
	    \nonumber \aa_{d-1}^{\transpose} \rr_{m,d-1} &= \left( \AA_{m,d-1} \bm{\phi}_{m,d-1} -\rr_{m,d-1} \right)^{\transpose} \rr_{m,d-1} \\
        \nonumber &=  \bm{\phi}_{m,d-1}^{\transpose} \AA_{m,d-1}^{\transpose} \rr_{m,d-1} -\rr_{m,d-1}^{\transpose} \rr_{m,d-1} \\
	    \label{eqn:k} &= - \| \rr_{m,d-1} \|^2,
    \end{align}
    as by \cref{eqn:exact_residual_defn} and the normal equations $\AA_{m,d-1}^\transpose \AA_{m,d-1} \bm{\phi}_{m,d-1} = \AA_{m,d-1}^\transpose \aa_{d-1}$.
    Equations \cref{eqn:ls_rec,eqn:k} imply
    \begin{align*}
        \ell_{m,d} (i) &= \ell_{m,d-1} (i) + \frac{1}{\| \rr_{m,d-1} \|^2}\bigg( \left( \rr_{m,d-1}(i) + \aa_{d-1}(i) \right)^2 \\
        &\quad  \quad \quad \quad \quad \quad\quad \quad \quad \quad \quad \quad \quad \quad -2 \aa_{d-1}(i)\left( \rr_{m,d-1} (i) + \aa_{d-1}(i) \right) + \aa_{d-1}(i)^2 \bigg) \\
        &= \ell_{m,d-1} (i) + \frac{(\rr_{m,d-1}(i))^2}{\| \rr_{m,d-1}\|^2} .
    \end{align*}
\end{proof}

%---------------------------------------------
\subsection{Proof of \cref{thm:rel_err} (Relative Errors for Approximate Leverage-Scores)}
%---------------------------------------------
To prove \cref{thm:rel_err}, we first introduce Lemmas \ref{lem:lev_scores_solns_min_prob}--\ref{lem:rhathat_bounded_by_r}.

%-----------------------------------------------------------
\begin{lemma}
    The $i^{th}$ leverage score of a matrix $\AA_{m,d} \in \mathbb{R}^{m \times d}$ is given by 
    \begin{equation*}
	    \ell_{m,d} (i) = \min_{\zz \in \mathbb{R}^m} \big\{ \|\zz\|^2 \: \, \mbox{s.t.} \: \, \AA_{m,d}^\transpose \zz = \AA_{m,d} (i,:) \big\},
    \end{equation*}
    for $i=1,\ldots,m$.
    \label{lem:lev_scores_solns_min_prob}
\end{lemma}

\begin{proof}
The Lagrangian function for the minimization problem is defined as
\begin{equation*}
	h(\zz,\bm{\lambda}) := \frac12 \zz^\transpose \zz -\bm{\lambda}^\transpose \left( \AA_{m,d}^\transpose \zz - \AA_{m,d}(i,:) \right). 
\end{equation*}
Because the objective function and the linear constraints are convex, by taking the first derivative with respect to the vector $\zz$ and setting equal to zero, we have
\begin{equation*}
	\frac{\partial h(\zz, \bm{\lambda})}{\partial \zz} = \zz - \AA_{m,d} \bm{\lambda} = 0 \implies \zz^\star = \AA_{m,d} \bm{\lambda}^\star.
\end{equation*}
By multiplying both sides of this expression by $\AA_{m,d}^\transpose$, we obtain
\begin{align*}
    \AA_{m,d}^\transpose \zz^\star &= \AA_{m,d}^\transpose \AA_{m,d} \bm{\lambda}^\star,
\end{align*}
which is equivalent to 
\begin{align*}
	\AA_{m,d} (i,:) &= \AA_{m,d}^\transpose \AA_{m,d} \bm{\lambda}^\star,
\end{align*}
implying
$$
    \bm{\lambda}^\star = (\AA_{m,d}^\transpose \AA_{m,d})^{-1} \AA_{m,d} (i,:).
$$
Thus,
$$
\zz^\star = \AA_{m,d} (\AA_{m,d}^\transpose \AA_{m,d})^{-1} \AA_{m,d} (i,:).
$$
Finally, we have
\begin{align*}
    \| \zz^\star \|^2 &= (\zz^\star)^\transpose (\zz^\star) \\
    &= \left( \AA_{m,d}^\transpose (i,:) (\AA_{m,d}^\transpose \AA_{m,d})^{-1} \AA_{m,d}^\transpose \right) \left( \AA_{m,d} (\AA_{m,d}^\transpose \AA_{m,d})^{-1} \AA_{m,d} (i,:) \right) \\
    &= \AA_{m,d}^\transpose (i,:) (\AA_{m,d}^\transpose \AA_{m,d})^{-1} \AA_{m,d} (i,:) \\
    &= \ell_{m,d} (i) .
\end{align*}
\end{proof}

%-----------------------------------------------------------
\begin{lemma}
    For $\AA_{m,d} \in \mathbb{R}^{m \times d}$, we have
    \begin{equation*}
    	\| \AA_{m,d}(i,:) \| \leq \| \AA_{m,d} \| \sqrt{\ell_{m,d}(i)}, \quad \text{for } i=1,\ldots,m.
    \end{equation*}
    \label{lem:matrix_row_norm_bound}
\end{lemma}

\begin{proof}
From the minimisation constraints in Lemma~\ref{lem:lev_scores_solns_min_prob}, we have
\begin{align*}
    \| \AA_{m,d}(i,:) \| &= \| \AA_{m,d}^\transpose \zz^\star \| \\
    &\leq \| \AA_{m,d}^\transpose \| \| \zz^\star \| \\
    &= \| \AA_{m,d} \| \sqrt{\ell_{m,d}(i)} .
\end{align*}
\end{proof}

%-----------------------------------------------------------
\begin{lemma}
    For $\AA_{m,d} \in \mathbb{R}^{m \times d}$ and $i=1,\ldots,m$, we have
    \begin{subequations}
    \begin{equation}
        | \rr_{m,d-1} (i) | \leq \sqrt{\| \bm{\phi}_{m,d-1} \|^2 +1} \, \| \AA_{m,d} \| \sqrt{ \ell_{m,d} (i) }, 
        \label{eqn:bound_resid_ith_entry}
    \end{equation}
    \begin{equation}
        | \hat{\rr}_{m,d-1} (i) | \leq \sqrt{\| \hat{\bm{\phi}}_{m,d-1} \|^2 +1} \, \| \AA_{m,d} \| \sqrt{ \ell_{m,d} (i) },
        \label{eqn:bound_resid_hat_ith_entry}
    \end{equation}
    \begin{equation}
        | \hat{\hat{\rr}}_{m,d-1}(i) | \leq (1+\varepsilon_2)\sqrt{\| \hat{\bm{\phi}}_{m,d-1} \|^2 +1} \| \AA_{m,d} \| \sqrt{ \ell_{m,d} (i) }
        \label{eqn:bound_resid_hat_hat_ith_entry}
    \end{equation}
    \end{subequations}
    where $\varepsilon_2$, $\rr_{m,d}$, $\hat{\rr}_{m,d}$, $\hat{\hat{\rr}}_{m,d}$, $\bm{\phi}_{m,d}$, $\hat{\bm{\phi}}_{m,d}$ and $\hat{\hat{\bm{\phi}}}_{m,d}$ are defined respectively in \cref{thm:matrix_mult_error_bound}, \cref{eqn:exact_residual_defn}, \cref{eqn:r_hat_defined}, \cref{eqn:r_hat_hat_defined}, \cref{eqn:definition_phi}, \cref{eqn:phi_hat_defined} and \cref{eqn:phi_hat_hat_defined}.
    \label{lem:residual_ith_entry_abs_value}
\end{lemma}

\begin{proof}
    By Lemma~\ref{lem:matrix_row_norm_bound}, the right-hand side of \cref{eqn:bound_resid_ith_entry} can be found as follows:
    \begin{align*}
        | \rr_{m,d-1}(i) | &= |\AA_{m,d-1}(i,:) \bm{\phi}_{m,d-1} - \aa_{d-1}(i) | \\
        &= | \left[ \AA_{m,d-1}(i,:) \: \: \aa_{d-1}(i) \right]  \left[ \bm{\phi}^\transpose_{m,d-1} \: \: -1 \right]^\transpose | \\
        &= | \AA_{m,d}(i,:) \left[ \bm{\phi}^\transpose_{m,d-1} \: \: -1 \right]^\transpose | \\
		&\leq \sqrt{\| \bm{\phi}_{m,d-1} \|^2 +1 } \, \| \AA_{m,d} (i,:) \| \\
		&\leq \sqrt{\| \bm{\phi}_{m,d-1} \|^2 +1} \, \| \AA_{m,d} \| \sqrt{ \ell_{m,d} (i) }.
    \end{align*}
    Inequality \cref{eqn:bound_resid_hat_ith_entry} can be obtained analogously by exchanging $\hat{\bm{\phi}}_{m,d-1}$ for $\bm{\phi}_{m,d-1}$. 
    
    Considering Inequality \cref{eqn:bound_resid_hat_hat_ith_entry}, from \cref{thm:matrix_mult_error_bound,lem:matrix_row_norm_bound,eqn:bound_resid_hat_ith_entry}, we have 
    \begin{align*}
        | \hat{\hat{\rr}}_{m,d-1}(i) | &= | \hat{\hat{\AA}}_{m,d-1}(i,:) \hat{\hat{\bm{\phi}}}_{m,d-1} - \aa_{d-1}(i) | \\
        &\leq |\hat{\hat{\AA}}_{m,d-1}(i,:) \hat{\hat{\bm{\phi}}}_{m,d-1} - \AA_{m,d-1}(i,:) \hat{\bm{\phi}}_{m,d-1}| + |\AA_{m,d-1}(i,:) \hat{\bm{\phi}}_{m,d-1} - \aa_{d-1}(i)| \\
        &\leq \varepsilon_2 \| \AA_{m,d-1}(i,:) \| \| \hat{\bm{\phi}}_{m,d-1} \| + | \hat{\rr}_{m,d-1}(i) | \\
        &\leq \varepsilon_2 \| \AA_{m,d-1} \| \sqrt{\ell_{m,d-1}(i)} \| \hat{\bm{\phi}}_{m,d-1} \| + \| \AA_{m,d} \| \sqrt{ \ell_{m,d} (i) } \sqrt{\| \hat{\bm{\phi}}_{m,d-1} \|^2 +1} \\
        &\leq \varepsilon_2 \| \AA_{m,d} \| \sqrt{\ell_{m,d}(i)} \sqrt{\| \hat{\bm{\phi}}_{m,d-1} \|^2 +1} + \| \AA_{m,d} \| \sqrt{ \ell_{m,d} (i) } \sqrt{\| \hat{\bm{\phi}}_{m,d-1} \|^2 +1} \\
        &= (1+\varepsilon_2) \| \AA_{m,d} \| \sqrt{ \ell_{m,d} (i) } \sqrt{\| \hat{\bm{\phi}}_{m,d-1} \|^2 +1}. 
    \end{align*}
\end{proof}

%-----------------------------------------------------------
\begin{lemma}
        For $\AA_{m,d} \in \mathbb{R}^{m,d}$, we have 
    \begin{equation*}
        | \rr_{m,d}(i) - \hat{\hat{\rr}}_{m,d}(i) | \leq \| \AA_{m,d}\| \sqrt{\ell_{m,d}(i) } \left( \sqrt{\varepsilon_1} \eta_{m,d} \| \bm{\phi}_{m,d} \| + \varepsilon_2 \| \hat{\bm{\phi}}_{m,d} \| \right),
    \end{equation*}
    for $i=1,\ldots,m$, where $\eta_{m,d}$ and $\varepsilon_1$ are both defined in \cref{thm:michael}, $\varepsilon_2$ is given in \cref{thm:matrix_mult_error_bound}, and $\bm{\phi}_{m,d}$, $\rr_{m,d}$, $\hat{\bm{\phi}}_{m,d}$, and $\hat{\hat{\rr}}_{m,d}$ are introduced respectively in \cref{eqn:definition_phi,eqn:exact_residual_defn,eqn:r_hat_hat_defined,eqn:phi_hat_defined}. \label{lem:abs_val_diff_r_rhathat}
\end{lemma}

\begin{proof}
	Firstly, consider $| \rr_{m,d}(i) - \hat{\rr}_{m,d}(i) |$ and use \cref{thm:michael,lem:matrix_row_norm_bound} to obtain
	\begin{align}
		\nonumber | \rr_{m,d}(i) - \hat{\rr}_{m,d}(i) | &= \left | (\AA_{m,d}(i,:) \bm{\phi}_{m,d} - \aa_{d}(i)) - (\AA_{m,d}(i,:) \hat{\bm{\phi}}_{m,d} - \aa_d(i)) \right | \\
		\nonumber &= \left | \AA_{m,d}(i,:) \bm{\phi}_{m,d} - \AA_{m,d}(i,:) \hat{\bm{\phi}}_{m,d} \right | \\
		\nonumber &\leq \| \AA_{m,d}(i,:) \| \| \bm{\phi}_{m,d} - \hat{\bm{\phi}}_{m,d} \|. \\
		\nonumber &\leq  \sqrt{\varepsilon_1} \eta_{m,d} \| \bm{\phi}_{m,d} \| \| \AA_{m,d}(i,:) \| \\
		\label{eqn:r-hat_r} &\leq \sqrt{\varepsilon_1} \eta_{m,d} \| \bm{\phi}_{m,d} \| \| \AA_{m,d} \| \sqrt{\ell_{m,d}(i)}. 
	\end{align}
	Secondly, consider $| \hat{\rr}_{m,d}(i) - \hat{\hat{\rr}}_{m,d}(i) |$ and use \cref{thm:matrix_mult_error_bound,lem:matrix_row_norm_bound} to obtain
	\begin{align}
    	\nonumber | \hat{\rr}_{m,d}(i) - \hat{\hat{\rr}}_{m,d}(i) | &= \left | (\AA_{m,d}(i,:) \hat{\bm{\phi}}_{m,d} - \aa_{d}(i)) - (\hat{\hat{\AA}}_{m,d}(i,:) \hat{\hat{\bm{\phi}}}_{m,d} - \aa_d(i)) \right | \\
	    \nonumber &= \left | \AA_{m,d}(i,:) \hat{\bm{\phi}}_{m,d} - \hat{\hat{\AA}}_{m,d}(i,:) \hat{\hat{\bm{\phi}}}_{m,d} \right | \\
	    \nonumber &\leq \varepsilon_2 \| \hat{\bm{\phi}}_{m,d} \|\| \AA_{m,d}(i,:) \|  \\
	    \label{eqn:hat_r-2hat_r} &\leq \varepsilon_2 \| \hat{\bm{\phi}}_{m,d} \| \| \AA_{m,d}\| \sqrt{\ell_{m,d}(i) }.
	\end{align}
	Now incorporating \cref{eqn:r-hat_r,eqn:hat_r-2hat_r}, we have
	\begin{align*}
	    | \rr_{m,d}(i) - \hat{\hat{\rr}}_{m,d}(i) | &\leq | \rr_{m,d}(i) - \hat{\rr}_{m,d}(i) | + | \hat{\rr}_{m,d}(i) - \hat{\hat{\rr}}_{m,d}(i) | \\
	    &\leq \sqrt{\varepsilon_1} \eta_{m,d} \| \bm{\phi}_{m,d} \| \| \AA_{m,d} \| \sqrt{ \ell_{m,d} (i) } + \varepsilon_2 \| \hat{\bm{\phi}}_{m,d} \| \| \AA_{m,d}\|  \sqrt{\ell_{m,d}(i) } \\
	    &= \| \AA_{m,d}\| \sqrt{\ell_{m,d}(i) } \left( \sqrt{\varepsilon_1} \eta_{m,d} \| \bm{\phi}_{m,d} \| + \varepsilon_2 \| \hat{\bm{\phi}}_{m,d} \| \right).
	\end{align*}
\end{proof}

%-----------------------------------------------------------
\begin{lemma}
	Suppose $\rr_{m,d}$ and $\hat{\rr}_{m,d}$ are defined as in \cref{eqn:exact_residual_defn,eqn:r_hat_defined}, respectively. Then, we have
	\begin{subequations}
	\begin{equation}
	    \| \rr_{m,d-1} \| \geq \sigma_{\min}(\AA_{m,d}) \sqrt{ (\| \bm{\phi}_{m,d-1} \|^2 +1)} ,
	\end{equation}
	\begin{equation}
	    \| \hat{\rr}_{m,d-1} \| \geq \sigma_{\min}(\AA_{m,d}) \sqrt{ (\| \hat{\bm{\phi}}_{m,d-1} \|^2 +1)} ,
	    \label{eqn:bound_norm_rhat}
	\end{equation}
	\end{subequations}
	    where $\sigma_{\min}(\cdot)$ denotes the minimum singular value, and $\bm{\phi}_{m,d}$, $\rr_{m,d}$, $\hat{\bm{\phi}}_{m,d}$, and $\hat{\rr}_{m,d}$ are given respectively in \cref{eqn:definition_phi,eqn:exact_residual_defn,eqn:phi_hat_defined,eqn:r_hat_defined}. 
		\label{lem:bound_norm_r_rhat}
\end{lemma}

\begin{proof}
By definition of the residual, we have
\begin{align*}
    \| \rr_{m,d-1} \| &= \| \AA_{m,d-1} \bm{\phi}_{m,d-1} - \aa_{d-1} \| \\
    &= \left \| \left[ \AA_{m,d-1} \: \: \aa_{d-1} \right] \left[ \bm{\phi}_{m,d-1}^\transpose \: \: -1 \right]^\transpose \right \| \\
    &= \left \| \AA_{m,d} \left[ \bm{\phi}_{m,d-1}^\transpose \: \: -1 \right]^\transpose \right \| \\
    &= \sqrt{\left[ \bm{\phi}_{m,d-1}^\transpose \: \: -1 \right] \AA_{m,d}^\transpose \AA_{m,d} \left[ \bm{\phi}_{m,d-1}^\transpose \: \: -1 \right]^\transpose} \\
    &= \sqrt{\frac{\left[ \bm{\phi}_{m,d-1}^\transpose \: \: -1 \right]}{\left\| \left[ \bm{\phi}_{m,d-1}^\transpose \: \: -1 \right] \right\|} \AA_{m,d}^\transpose \AA_{m,d} \frac{\left[ \bm{\phi}_{m,d-1}^\transpose \: \: -1 \right]^\transpose}{\left\| \left[ \bm{\phi}_{m,d-1}^\transpose \: \: -1 \right] \right\|}} \left\| \left[ \bm{\phi}_{m,d-1}^\transpose \: \: -1 \right] \right\|^2 \\
    &\geq \sigma_{\min}(\AA_{m,d}) \sqrt{ \| \bm{\phi}_{m,d-1} \|^2 +1 }.
\end{align*}
Inequality \cref{eqn:bound_norm_rhat} is shown analogously. 
\end{proof}

%-----------------------------------------------------------
\begin{lemma}
Suppose $\hat{\rr}_{m,d}$ and $\hat{\hat{\rr}}_{m,d}$ are defined as in \cref{eqn:r_hat_defined,eqn:r_hat_hat_defined}, respectively. Then, we have
\begin{equation*}
    \| \hat{\rr}_{m,d} - \hat{\hat{\rr}}_{m,d} \| \leq \varepsilon_2 \kappa(\AA_{m,d}) \sqrt{d} \| \hat{\rr}_{m,d} \|,
\end{equation*}
where $\kappa(\cdot)$ denotes the condition number, and $\varepsilon_2$ is as given in \cref{thm:matrix_mult_error_bound}.
\label{lem:norm_rhat_minus_rhathat}
\end{lemma}

\begin{proof}
	Using \cref{lem:bound_norm_r_rhat,eqn:r_hat_defined,eqn:r_hat_hat_defined,thm:matrix_mult_error_bound} we obtain
	\begin{align*}
		\| \hat{\rr}_{m,d} - \hat{\hat{\rr}}_{m,d} \| &= \| \hat{\rr}_{m,d} - \hat{\hat{\rr}}_{m,d} \|_F \\
		&= \| (\AA_{m,d} \hat{\bm{\phi}}_{m,d} - \aa_{d}) - (\hat{\hat{\AA}}_{m,d} \hat{\hat{\bm{\phi}}}_{m,d} - \aa_{d}) \|_F \\
		&= \| \AA_{m,d} \hat{\bm{\phi}}_{m,d} - \hat{\hat{\AA}}_{m,d} \hat{\hat{\bm{\phi}}}_{m,d} \|_F \\
	    &\leq \varepsilon_2 \| \AA_{m,d} \|_F \left\| \hat{\bm{\phi}}_{m,d} \right\| \\
	    &\leq \varepsilon_2 \| \AA_{m,d} \|_F \sqrt{ \| \hat{\bm{\phi}}_{m,d} \|^2 +1 } \\
	    &\leq \varepsilon_2 \sqrt{d}\| \AA_{m,d} \| \frac{\left\| \hat{\bm{\rr}}_{m,d} \right\|}{\sigma_{\min}(\AA_{m,d})} \\
		&\leq \varepsilon_2 \kappa(\AA_{m,d}) \sqrt{d} \left\| \hat{\rr}_{m,d} \right\| .
	\end{align*}
\end{proof}

%-----------------------------------------------------------
\begin{lemma}
    Suppose $\rr_{m,d}$ and $\hat{\hat{\rr}}_{m,d}$ are defined as in \cref{eqn:exact_residual_defn,eqn:r_hat_hat_defined}, respectively. Then, we have
    \begin{equation*}
        \| \hat{\hat{\rr}}_{m,d} \| \leq ( 1+\varepsilon^\star \xi_{m,d} ) \| \rr_{m,d} \|,
    \end{equation*}
    where $\varepsilon^{\star} = \max\{\varepsilon_1,\varepsilon_2\}$, and $\varepsilon_1$ and $\varepsilon_2$ are as in \cref{thm:michael,thm:matrix_mult_error_bound} and $ \xi_{m,d} $ is as in \cref{thm:rel_err}.
    \label{lem:rhathat_bounded_by_r}
\end{lemma} 

\begin{proof}
    From \cref{lem:norm_rhat_minus_rhathat}, we obtain
    \begin{align}
      	\nonumber \left | \| \hat{\rr}_{m,d} \| - \| \hat{\hat{\rr}}_{m,d} \| \right| &\leq \| \hat{\rr}_{m,d} - \hat{\hat{\rr}}_{m,d} \| \\
      	\label{eqn:hat_t-2hat_r_2} &\leq \varepsilon_2 \kappa(\AA_{m,d}) \sqrt{d} \| \hat{\rr}_{m,d} \| . 
    \end{align}
    By rearranging \cref{eqn:hat_t-2hat_r_2} and using \cref{thm:michael}, we have
    \begin{align*}
        \| \hat{\hat{\rr}}_{m,d} \| &\leq  (1+\varepsilon_2 \kappa(\AA_{m,d}) \sqrt{d}) \| \hat{\rr}_{m,d} \| \\ 
        &\leq (1+\varepsilon_2 \kappa(\AA_{m,d}) \sqrt{d}) (1+\varepsilon_1) \| \rr_{m,d} \| \\
        &= \left ( 1+\varepsilon_1+\varepsilon_2 \kappa(\AA_{m,d}) \sqrt{d})  + \varepsilon_1 \varepsilon_2 \kappa(\AA_{m,d}) \sqrt{d} \right ) \| \rr_{m,d} \| \\
        &\leq \left ( 1+\varepsilon^\star+\varepsilon^\star \kappa(\AA_{m,d}) \sqrt{d})  + \varepsilon^\star \kappa(\AA_{m,d}) \sqrt{d} \right ) \| \rr_{m,d} \| \\
        &=
        \left( 1 + \varepsilon^\star(1+2\kappa(\AA_{m,d}) \sqrt{d}) \right) \| \rr_{m,d}\| \\
        &= ( 1 + \varepsilon^\star \xi_{m,d} ) \| \rr_{m,d}\|.
    \end{align*}
\end{proof}

%------------------------------------
\subsubsection*{Proof of \cref{thm:rel_err}}
%------------------------------------
\begin{proof}
We prove this theorem by induction for the two cases $1 \leq d \leq s_2$ and $s_2 \leq d < n$.

\paragraph{Case 1: $\bm{1 \leq d \leq s_2}$.} The statement is trivial for $d=1$, as $\hat{\ell}_{m,1} = \ell_{m,1}$. Now let us assume the statement of \cref{thm:rel_err} is true for all $d < \bar{d}$ for some $\bar{d} \leq s_2$, and prove it is also true for $d = \bar{d}$.

Using the induction hypothesis along with \cref{thm:General_Lev_Score_Recursion,def:Defn_Two}, we have
\begin{align*}
	%-----------
    \left | \ell_{m,\bar{d}}(i) - \hat{\ell}_{m,\bar{d}}(i) \right | &= \left | \ell_{m,\bar{d}-1}(i) + \frac{(\rr_{m,\bar{d}-1}(i))^2}{\| \rr_{m,\bar{d}-1} \| ^2} -  \hat{\ell}_{m,\bar{d}-1}(i) - \frac{(\hat{\rr}_{m,\bar{d}-1}(i))^2}{\| \hat{\rr}_{m,\bar{d}-1} \| ^2} \right | \\
	%-----------
    &\leq \left| \ell_{m,\bar{d}-1}(i) - \hat{\ell}_{m,\bar{d}-1}(i) \right| \\
    &\quad \quad + \left | \frac{(\rr_{m,\bar{d}-1}(i))^2}{\| \rr_{m,\bar{d}-1} \| ^2} - \frac{(\rr_{m,\bar{d}-1}(i))^2}{\| \hat{\rr}_{m,\bar{d}-1} \| ^2} + \frac{(\rr_{m,\bar{d}-1}(i))^2}{\| \hat{\rr}_{m,\bar{d}-1} \| ^2}  - \frac{(\hat{\rr}_{m,\bar{d}-1}(i))^2}{\| \hat{\rr}_{m,\bar{d}-1} \| ^2} \right | \\
	%-----------
	&\leq \big( \sqrt{\xi_{m,\bar{d}-2}} + 5(\eta_{m,\bar{d}-2}+2) \kappa^2 (\AA_{m,\bar{d}-1}) \big) (\bar{d}-2) \sqrt{\varepsilon^\star} \ell_{m,\bar{d}-1}(i) \\
    &\quad \quad + (\rr_{m,\bar{d}-1}(i))^2 \left | \frac{1}{\| \rr_{m,\bar{d}-1} \| ^2} - \frac{1}{\| \hat{\rr}_{m,\bar{d}-1} \|^2} \right | + \frac{1}{\| \hat{\rr}_{m,\bar{d}-1} \| ^2} \left| (\rr_{m,\bar{d}-1}(i))^2 - (\hat{\rr}_{m,\bar{d}-1}(i))^2 \right |.
	%-----------
\end{align*}
Now, by using \cref{thm:michael}, we have
\begin{align*}
	%-----------
    \left| \ell_{m,\bar{d}}(i) - \hat{\ell}_{m,\bar{d}}(i) \right| &\leq \big( \sqrt{\xi_{m,\bar{d}-2}} + 5(\eta_{m,\bar{d}-2}+2) \kappa^2 (\AA_{m,\bar{d}-1}) \big) (\bar{d}-2) \sqrt{\varepsilon^\star} \ell_{m,\bar{d}-1}(i) \\ 
    &\quad \quad + (\rr_{m,\bar{d}-1}(i))^2 \left | \frac{1}{\| \rr_{m,\bar{d}-1} \| ^2} - \frac{1}{(1+\varepsilon_1)^2 \| \rr_{m,\bar{d}-1} \|^2} \right | \\
    &\quad \quad + \frac{1}{\| \rr_{m,\bar{d}-1} \| ^2} \left| (\rr_{m,\bar{d}-1}(i) - \hat{\rr}_{m,\bar{d}-1}(i)) \right | \left| (\rr_{m,\bar{d}-1}(i) + \hat{\rr}_{m,\bar{d}-1}(i)) \right | \\
	%-----------
    &\leq \big( \sqrt{\xi_{m,\bar{d}-2}} + 5(\eta_{m,\bar{d}-2}+2) \kappa^2 (\AA_{m,\bar{d}-1}) \big) (\bar{d}-2) \sqrt{\varepsilon^\star} \ell_{m,\bar{d}-1}(i) + \frac{\varepsilon_1^2+2\varepsilon_1}{(1+\varepsilon_1)^2} \frac{(\rr_{m,\bar{d}-1}(i))^2}{\| \rr_{m,\bar{d}-1} \|^2} \\
    &\quad \quad + \frac{1}{\| \rr_{m,\bar{d}-1} \| ^2} \left| (\rr_{m,\bar{d}-1}(i) - \hat{\rr}_{m,\bar{d}-1}(i)) \right | \left( \left| \rr_{m,\bar{d}-1}(i) \right | + \left| \hat{\rr}_{m,\bar{d}-1}(i) \right| \right). 
	%-----------
\end{align*}
From \cref{thm:General_Lev_Score_Recursion}, we have
\begin{subequations}
\begin{align}
	\label{eqn:ell1} \ell_{m,\bar{d}-1}(i) &\leq \ell_{m,\bar{d}}(i) \text{ and} \\
	\label{eqn:ell2} \frac{(\rr_{m,\bar{d}-1}(i))^2}{\| \rr_{m,\bar{d}-1} \|^2} &\leq \ell_{m,\bar{d}}(i) 
\end{align}
\end{subequations}
for $i=1,\ldots,m$ and $2 \leq \bar{d} \leq s_2$. Considering \cref{eqn:ell1,eqn:ell2} as well as using \cref{eqn:r-hat_r,lem:residual_ith_entry_abs_value}, we now have
\begin{align*}
	%-----------
    \left | \ell_{m,\bar{d}}(i) - \hat{\ell}_{m,\bar{d}}(i) \right | &\leq \big( \sqrt{\xi_{m,\bar{d}-2}} + 5(\eta_{m,\bar{d}-2}+2) \kappa^2 (\AA_{m,\bar{d}-1}) \big) (\bar{d}-2) \sqrt{\varepsilon^\star} \ell_{m,\bar{d}}(i) + \frac{\varepsilon_1^2+2\varepsilon_1}{(1+\varepsilon_1)^2} \ell_{m,\bar{d}}(i) \\
    &\quad \quad + \frac{1}{\| \rr_{m,\bar{d}-1} \| ^2} \left( \sqrt{\varepsilon_1} \eta_{m,\bar{d}-1} \| \bm{\phi}_{m,\bar{d}-1} \| \| \AA_{m,\bar{d}-1} \| \sqrt{ \ell_{m,\bar{d}-1} (i) } \right) \\
    &\quad \quad \quad \quad \times \bigg( \sqrt{\| \bm{\phi}_{m,\bar{d}-1} \|^2 +1} \, \| \AA_{m,\bar{d}} \| \sqrt{ \ell_{m,\bar{d}} (i) } +  \sqrt{\| \hat{\bm{\phi}}_{m,\bar{d}-1} \|^2 +1} \, \| \AA_{m,\bar{d}} \| \sqrt{ \ell_{m,\bar{d}} (i) } \bigg) \\
	%-----------
    &\leq \big( \sqrt{\xi_{m,\bar{d}-2}} + 5(\eta_{m,\bar{d}-2}+2) \kappa^2 (\AA_{m,\bar{d}-1}) \big) (\bar{d}-2) \sqrt{\varepsilon^\star} \ell_{m,\bar{d}-1}(i) +  + \bigg( \frac{\sqrt{\varepsilon_1}(2+\varepsilon_1)}{(1+\varepsilon_1)^2}  \\ 
    &+ \frac{ \eta_{m,\bar{d}-1} \| \bm{\phi}_{m,\bar{d}-1} \| \| \AA_{m,\bar{d}-1} \| \| \AA_{m,\bar{d}} \|}{\| \rr_{m,\bar{d}-1} \| ^2} \bigg( \sqrt{\| \bm{\phi}_{m,\bar{d}-1} \|^2 +1} + \sqrt{\| \hat{\bm{\phi}}_{m,\bar{d}-1} \|^2 +1} \bigg) \bigg) \sqrt{\varepsilon_1} \ell_{m,\bar{d}}(i).
	%-----------
\end{align*}
One can easily see that for a given $0 < \varepsilon_1 < 1$, we have
\begin{equation}
	\frac{\sqrt{\varepsilon_1}(2+\varepsilon_1)}{(1+\varepsilon_1)^2} \leq \frac{\sqrt{\varepsilon_1}(2+2\varepsilon_1)}{(1+\varepsilon_1)^2}
	= \frac{2\sqrt{\varepsilon_1}}{1+\varepsilon_1} \leq 1.
	\label{eq:frac_eps1}
\end{equation}
Additionally, $\| \AA_{m,\bar{d}-1} \| \leq \| \AA_{m,\bar{d}} \|$ and $ \kappa(\AA_{m,\bar{d}-1}) \leq \kappa(\AA_{m,\bar{d}}) $ (equivalently, $\eta_{m,\bar{d}-1} \leq \eta_{m,\bar{d}}$) for all $2 \leq \bar{d} \leq s_2$, and using \cref{thm:michael,lem:bound_norm_r_rhat}, we can obtain
\begin{align*}
	%-----------
    \left | \ell_{m,\bar{d}}(i) - \hat{\ell}_{m,\bar{d}}(i) \right | &\leq \big( \sqrt{\xi_{m,\bar{d}-2}} + 5(\eta_{m,\bar{d}-1}+2) \kappa^2 (\AA_{m,\bar{d}}) \big) (\bar{d}-2) \sqrt{\varepsilon^\star} \ell_{m,\bar{d}}(i) \\
	&\quad \quad + \bigg( 1 + \frac{ \eta_{m,\bar{d}-1} \sqrt{\| \bm{\phi}_{m,\bar{d}-1} \|^2 + 1} \| \AA_{m,\bar{d}} \|^2 }{\| \rr_{m,\bar{d}-1} \| } \\
    &\quad \quad \quad \quad \times \bigg( \frac{\sqrt{\| \bm{\phi}_{m,\bar{d}-1} \|^2 +1}}{\| \rr_{m,\bar{d}-1} \| } + \frac{\sqrt{\| \hat{\bm{\phi}}_{m,\bar{d}-1} \|^2 +1}}{\| \rr_{m,\bar{d}-1} \| } \bigg) \bigg) \sqrt{\varepsilon^\star} \ell_{m,\bar{d}}(i) \\  
	%-----------
	&\leq \big( \sqrt{\xi_{m,\bar{d}-2}} + 5(\eta_{m,\bar{d}-1}+2) \kappa^2 (\AA_{m,\bar{d}}) \big) (\bar{d}-2) \sqrt{\varepsilon^\star} \ell_{m,\bar{d}}(i) \\
	&\quad \quad + \bigg( 1 + \frac{ \eta_{m,\bar{d}-1} \sqrt{\| \bm{\phi}_{m,\bar{d}-1} \|^2 + 1} \| \AA_{m,\bar{d}} \|^2 }{\| \rr_{m,\bar{d}-1} \| } \\
    &\quad \quad \quad \quad \times \bigg( \frac{\sqrt{\| \bm{\phi}_{m,\bar{d}-1} \|^2 +1}}{\| \rr_{m,\bar{d}-1} \| } + \frac{(1+\varepsilon_1)\sqrt{\| \hat{\bm{\phi}}_{m,\bar{d}-1} \|^2 +1}}{\| \hat{\rr}_{m,\bar{d}-1} \| } \bigg) \bigg) \sqrt{\varepsilon^\star} \ell_{m,\bar{d}}(i) \\
	%-----------
	&\leq \big( \sqrt{\xi_{m,\bar{d}-2}} + 5(\eta_{m,\bar{d}-1}+2) \kappa^2 (\AA_{m,\bar{d}}) \big) (\bar{d}-2) \sqrt{\varepsilon^\star} \ell_{m,\bar{d}}(i) \\
	&\quad \quad + \bigg( 1 + \frac{ \eta_{m,\bar{d}-1} \| \AA_{m,\bar{d}} \|^2 }{ \sigma_{\min}(\AA_{m,\bar{d}}) } \bigg( \frac{ 1 }{ \sigma_{\min}(\AA_{m,\bar{d}}) } + \frac{ 1+\varepsilon_1 }{ \sigma_{\min}(\AA_{m,\bar{d}}) } \bigg) \bigg) \sqrt{\varepsilon^\star} \ell_{m,\bar{d}}(i) \\
	%-----------
	&\leq \big( \sqrt{\xi_{m,\bar{d}-2}} + 5(\eta_{m,\bar{d}-1}+2) \kappa^2 (\AA_{m,\bar{d}}) \big) (\bar{d}-2) \sqrt{\varepsilon^\star} \ell_{m,\bar{d}}(i) \\
	&\quad \quad + \big( 1 + 3\eta_{m,\bar{d}-1} \kappa^2 (\AA_{m,\bar{d}}) \big) \sqrt{\varepsilon^\star} \ell_{m,\bar{d}}(i).
	%-----------
\end{align*}
One can easily see $ 1 < \xi_{m,\bar{d}-2} \leq \xi_{m,\bar{d}-1} $, implying
\begin{align*}
	%-----------
	\left | \ell_{m,\bar{d}}(i) - \hat{\ell}_{m,\bar{d}}(i) \right | &\leq
	\big( \sqrt{\xi_{m,\bar{d}-1}} + 5(\eta_{m,\bar{d}-1}+2) \kappa^2 (\AA_{m,\bar{d}}) \big) (\bar{d}-1) \sqrt{\varepsilon^\star} \ell_{m,\bar{d}}(i). 
	%-----------
\end{align*}
Thus, \cref{thm:rel_err} holds for $1 \leq d \leq s_2$ by induction. 

\paragraph{Case 2: $\bm{s_2 \leq d < n}$.} Now we suppose \cref{thm:rel_err} holds for all $ d < \bar{d} $ for some $ s_2 \leq \bar{d} < n $, and prove it is also true for $ d = \bar{d} $. 

Analogous to Case 1 and using the induction hypothesis along with \cref{thm:General_Lev_Score_Recursion,def:Defn_Two}, we have
\begin{align*}
	%-----------
	\left | \ell_{m,\bar{d}}(i) - \hat{\ell}_{m,\bar{d}}(i) \right | &\leq \big( \sqrt{\xi_{m,\bar{d}-2}} + 5(\eta_{m,\bar{d}-2}+2) \kappa^2 (\AA_{m,\bar{d}-1}) \big) (\bar{d}-2) \sqrt{\varepsilon^\star} \ell_{m,\bar{d}-1}(i) \\
	&\quad \quad + (\rr_{m,\bar{d}-1}(i))^2 \left | \frac{1}{\| \rr_{m,\bar{d}-1} \| ^2} - \frac{1}{\| \hat{\hat{\rr}}_{m,\bar{d}-1} \|^2} \right | \\
	&\quad \quad + \frac{1}{\| \hat{\hat{\rr}}_{m,\bar{d}-1} \| ^2} \left| (\rr_{m,\bar{d}-1}(i))^2 - (\hat{\hat{\rr}}_{m,\bar{d}-1}(i))^2 \right |.
	%-----------
\end{align*}
Using \cref{lem:rhathat_bounded_by_r,lem:residual_ith_entry_abs_value,lem:abs_val_diff_r_rhathat,lem:bound_norm_r_rhat}, we obtain
\begin{align*}    
	%-----------
    | \ell_{m,\bar{d}}(i) - \hat{\ell}_{m,\bar{d}}(i) | &\leq  \big( \sqrt{\xi_{m,\bar{d}-2}} + 5(\eta_{m,\bar{d}-2}+2) \kappa^2 (\AA_{m,\bar{d}-1}) \big) (\bar{d}-2) \sqrt{\varepsilon^\star} \ell_{m,\bar{d}-1}(i) \\
    &\quad \quad + (\rr_{m,\bar{d}-1}(i))^2 \left | \frac{1}{\| \rr_{m,\bar{d}-1} \|^2} - \frac{1}{\left( 1+\varepsilon^\star \xi_{m,\bar{d}-1} \right)^2\| \rr_{m,\bar{d}-1} \|^2} \right | \\
    &\quad \quad + \frac{1}{\| \rr_{m,\bar{d}-1} \|^2} \left | \rr_{m,\bar{d}-1}(i) - \hat{\hat{\rr}}_{m,\bar{d}-1}(i) \right | \left | \rr_{m,\bar{d}-1}(i) + \hat{\hat{\rr}}_{m,\bar{d}-1}(i) \right | \\
	%-----------
    &\leq \big( \sqrt{\xi_{m,\bar{d}-2}} + 5(\eta_{m,\bar{d}-2}+2) \kappa^2 (\AA_{m,\bar{d}-1}) \big) (\bar{d}-2) \sqrt{\varepsilon^\star} \ell_{m,\bar{d}-1}(i) \\
    &\quad \quad + \left( \frac{ ({\varepsilon^\star} \xi_{m,\bar{d}-1})^2 +2\varepsilon^\star \xi_{m,\bar{d}-1} }{(1+\varepsilon^\star \xi_{m,\bar{d}-1})^2} \right) \frac{(\rr_{m,\bar{d}-1}(i))^2}{\| \rr_{m,\bar{d}-1} \|^2} \\
    &\quad \quad + \frac{1}{\| \rr_{m,\bar{d}-1}\|^2} \| \AA_{m,\bar{d}-1} \| \sqrt{\ell_{m,\bar{d}-1}(i)}  \left( \sqrt{\varepsilon_1} \eta_{m,\bar{d}-1} \| \bm{\phi}_{m,\bar{d}-1}\| + \varepsilon_2 \| \hat{\bm{\phi}}_{m,\bar{d}-1} \| \right) \\
	& \times \bigg( \sqrt{\|\bm{\phi}_{m,\bar{d}-1}\|^2 + 1} \, \| \AA_{m,\bar{d}} \| \sqrt{\ell_{m,\bar{d}}(i)} + (1+\varepsilon_2) \sqrt{\|\hat{\bm{\phi}}_{m,\bar{d}-1}\|^2 + 1} \, \| \AA_{m,\bar{d}} \| \sqrt{\ell_{m,\bar{d}}(i)} \bigg) \\
	%-----------
    &\leq \big( \sqrt{\xi_{m,\bar{d}-2}} + 5(\eta_{m,\bar{d}-2}+2) \kappa^2 (\AA_{m,\bar{d}-1}) \big) (\bar{d}-2) \sqrt{\varepsilon^\star} \ell_{m,\bar{d}-1}(i) \\
    &\quad \quad + \left( \frac{ ({\varepsilon^\star} \xi_{m,\bar{d}-1})^2 +2\varepsilon^\star \xi_{m,\bar{d}-1} }{(1+\varepsilon^\star \xi_{m,\bar{d}-1})^2} \right) \ell_{m,\bar{d}-1}(i) \\
    &\quad \quad + \| \AA_{m,\bar{d}-1} \| \sqrt{\ell_{m,\bar{d}-1}(i)} \left( \sqrt{\varepsilon_1} \eta_{m,\bar{d}-1} \frac{\sqrt{\| \bm{\phi}_{m,\bar{d}-1}\|^2 +1}}{\| \rr_{m,\bar{d}-1}\|} + \varepsilon_2 \frac{\sqrt{\| \hat{\bm{\phi}}_{m,\bar{d}-1} \|^2 +1}}{\| \rr_{m,\bar{d}-1}\|} \right) \\
	& \times \bigg( \frac{\sqrt{\| \bm{\phi}_{m,\bar{d}-1}\|^2 +1}}{\| \rr_{m,\bar{d}-1}\|} \, \| \AA_{m,\bar{d}} \| \sqrt{\ell_{m,\bar{d}}(i)} + (1+\varepsilon_2) \frac{\sqrt{\| \hat{\bm{\phi}}_{m,\bar{d}-1} \|^2 +1}}{\| \rr_{m,\bar{d}-1}\|} \, \| \AA_{m,\bar{d}} \| \sqrt{\ell_{m,\bar{d}}(i)} \bigg) \\
	%-----------
	&\leq \big( \sqrt{\xi_{m,\bar{d}-2}} + 5(\eta_{m,\bar{d}-1}+2) \kappa^2 (\AA_{m,\bar{d}}) \big) (\bar{d}-2) \sqrt{\varepsilon^\star} \ell_{m,\bar{d}}(i) \\
	&\quad \quad + \left( \frac{ ({\varepsilon^\star} \xi_{m,\bar{d}-1})^2 +2\varepsilon^\star \xi_{m,\bar{d}-1} }{(1+\varepsilon^\star \xi_{m,\bar{d}-1})^2} \right) \ell_{m,\bar{d}-1}(i) \\
	&\quad \quad + \| \AA_{m,\bar{d}} \|^2 \left( \sqrt{\varepsilon_1} \eta_{m,\bar{d}-1} \frac{ 1 }{ \sigma_{\min}(\AA_{m,\bar{d}}) } + \varepsilon_2 (1 + \varepsilon_1) \frac{ 1 }{ \sigma_{\min}(\AA_{m,\bar{d}}) } \right) \\
	&\quad \quad \quad \quad \times \bigg( \frac{ 1 }{ \sigma_{\min}(\AA_{m,\bar{d}}) } + (1+\varepsilon_2) (1 + \varepsilon_1) \frac{ 1 }{ \sigma_{\min}(\AA_{m,\bar{d}}) } \bigg) \ell_{m,\bar{d}}(i) \\
	%-----------
	&\leq \big( \sqrt{\xi_{m,\bar{d}-2}} + 5(\eta_{m,\bar{d}-1}+2) \kappa^2 (\AA_{m,\bar{d}}) \big) (\bar{d}-2) \sqrt{\varepsilon^\star} \ell_{m,\bar{d}}(i) \\
	&\quad \quad + \left( \frac{ 2\sqrt{\varepsilon^\star \xi_{m,\bar{d}-1}} }{ 1+\varepsilon^\star \xi_{m,\bar{d}-1} } \right) \sqrt{\xi_{m,\bar{d}-1}} \sqrt{\varepsilon^\star} \ell_{m,\bar{d}}(i) \\
	&\quad \quad + \| \AA_{m,\bar{d}} \|^2 \left( \sqrt{\varepsilon^\star} \eta_{m,\bar{d}-1} \frac{ 1 }{ \sigma_{\min}(\AA_{m,\bar{d}}) } + \varepsilon^\star (1 + \varepsilon^\star) \frac{ 1 }{ \sigma_{\min}(\AA_{m,\bar{d}}) } \right) \\
	&\quad \quad \quad \quad \times \bigg( \frac{ 1 }{ \sigma_{\min}(\AA_{m,\bar{d}}) } + (1+\varepsilon^\star) (1 + \varepsilon^\star) \frac{ 1 }{ \sigma_{\min}(\AA_{m,\bar{d}}) } \bigg) \ell_{m,\bar{d}}(i) \\
	%-----------
	&\leq \big( \sqrt{\xi_{m,\bar{d}-2}} + 5(\eta_{m,\bar{d}-1}+2) \kappa^2 (\AA_{m,\bar{d}}) \big) (\bar{d}-2) \sqrt{\varepsilon^\star} \ell_{m,\bar{d}}(i) \\
	&\quad \quad + \bigg( \sqrt{\xi_{m,\bar{d}-1}} + 5 \left( \eta_{m,\bar{d}-1} + 2 \right) \kappa(\AA_{m,\bar{d}})^2 \bigg) \sqrt{\varepsilon^\star} \ell_{m,\bar{d}}(i).
	%-----------
\end{align*}
$\xi_{m,d} $ is an increasing function in $ d $, we have 
	\begin{align*}    
		%-----------
		| \ell_{m,\bar{d}}(i) - \hat{\ell}_{m,\bar{d}}(i) | &\leq \big( \sqrt{\xi_{m,\bar{d}-1}} + 5(\eta_{m,\bar{d}-1}+2) \kappa^2 (\AA_{m,\bar{d}}) \big) (\bar{d}-1) \sqrt{\varepsilon^\star} \ell_{m,\bar{d}}(i).
		%-----------
	\end{align*}
\end{proof}

%\section{More Plots}
%\luke{Room for extra plots that we may want to include that have been removed from the paper, e.g. leverage score distribution for SALSA experiments.}

%\end{appendix}

	%-----------------------------------------------
	%\section{References}
	%-----------------------------------------------
	
	\bibliographystyle{plain}
	\bibliography{SALSA_Bib}
	
	%-----------------------------------------------

\end{document}